\documentclass[10pt]{article}
\usepackage{amsmath, natbib}
\usepackage[top=1in, bottom=1in, left=1in, right=1in]{geometry}
\usepackage{amssymb}
\usepackage{amsthm}
\usepackage{amscd}
\usepackage{amsfonts}
\usepackage{mathtools}
\usepackage{graphicx}%
\usepackage{caption,tabularx}
\usepackage{hyperref}
\hypersetup{colorlinks,
            linkcolor=blue,
            citecolor=blue,
            urlcolor=magenta,
            linktocpage,
            plainpages=false}
\bibpunct{(}{)}{;}{a}{,}{,}
\newtheorem{theorem}{Theorem}
\newtheorem{lemma}{Lemma}
\newtheorem{definition}{Definition}

\usepackage{booktabs} 
\usepackage[ruled]{algorithm2e}
\usepackage{tikz}
\usetikzlibrary{automata,arrows}
\usepackage{adjustbox}

\SetKwProg{Fn}{function}{}{end}
\SetKwProg{Init}{initialization}{}{}
\SetKwInput{Input}{input}
\SetKw{KwTo}{in}
\SetKw{KwWhere}{where}

\usepackage{setspace}
\usepackage[normalem]{ulem}

\newcommand{\minihead}[1]{{\vspace{.45em}\noindent\textbf{#1.} }}

\DeclareMathOperator*{\argmin}{arg\,min}

\newcommand{\normsq}[1]{\| #1 \|_2^2}
\newcommand{\norm}[1]{{\left\| #1 \right\|_{2}}}

\newcommand{\lpnorm}[2]{\left\| #1 \right\|_{#2}}
\newcommand*{\E}{\mathbb{E}}

\renewcommand{\cite}[1]{\citep{#1}}
\newcommand{\lonenorm}[1]{{\| #1 \|_{1}}}
\let\b\relax
\newcommand{\b}[1]{\mathbf{#1}}

\begin{document}
\title{Sketching Linear Classifiers over Data Streams}

\author{Kai Sheng Tai, Vatsal Sharan, Peter Bailis, Gregory Valiant \\[0.25ex] \texttt{\{kst, vsharan, pbailis, valiant\}@cs.stanford.edu} \\[0.5ex] Stanford University}

\date{}

%
%


\maketitle

\begin{abstract}
We introduce a new sub-linear space sketch---the Weight-Median
Sketch---for learning compressed linear classifiers over data streams while supporting
the efficient recovery of large-magnitude weights in the model.
This enables memory-limited execution of several statistical analyses over streams, including
online feature selection, streaming data explanation, relative deltoid
detection, and streaming estimation of pointwise mutual
information. Unlike related sketches that capture the most
frequently-occurring features (or items) in a data stream, the
Weight-Median Sketch captures the features that are most
discriminative of one stream (or class) compared to another. The
Weight-Median Sketch adopts the core data structure used in the
Count-Sketch, but, instead of sketching counts, it captures sketched
gradient updates to the model parameters. We provide a theoretical
analysis that establishes recovery guarantees for
batch and online learning, and demonstrate empirical improvements in
memory-accuracy trade-offs over alternative memory-budgeted methods,
including count-based sketches and feature hashing.

\end{abstract}

\section{Introduction}

With the rapid growth of streaming data volumes, memory-efficient
sketches are an increasingly important tool in analytics tasks
such as finding frequent
items~\citep{charikar2002finding,cormode2005improved,metwally2005efficient,larsen2016heavy},
estimating quantiles~\citep{greenwald2001space,luo2016quantiles}, and approximating the
number of distinct items~\citep{flajolet1985approximate}. Sketching algorithms trade off between
space utilization and approximation accuracy, and are therefore well suited to settings where
memory is scarce or where highly-accurate estimation is not essential.
For example, sketches are used in measuring traffic statistics on resource-constrained network switch hardware \citep{yu2013software}
and in processing approximate aggregate queries in sensor networks \citep{considine2004approximate}.
Moreover, even in commodity server environments where memory is more plentiful, sketches are useful as a lightweight means to perform approximate analyses like identifying frequent search queries or URLs within a broader stream processing pipeline \citep{boykin2014summingbird}.

Machine learning is applicable in many of the same resource-constrained deployment scenarios as existing sketching algorithms.
With the widespread adoption of mobile devices, wearable electronics, and smart home appliances,
there is increasing interest in memory-constrained learning, where statistical models on these devices are updated on-the-fly in response to locally-observed data \citep{longstaff2010improving,kapoor2012memory,mcgraw2016personalized,smith2017federated}. These online updates allow ML-enabled systems to adapt to individual users or local environments.
For example, language models on mobile devices can be personalized in order to improve the accuracy of
speech recognition systems \citep{mcgraw2016personalized},
mobile facial recognition systems can be updated based on user supervision \citep{kapoor2012memory},
packet filters on network routers can be incrementally improved \citep{vamanan2010efficuts,dainotti2012issues},
and human activity classifiers can be tailored to individual motion patterns for more accurate classification \citep{longstaff2010improving,yu2016hybridizing}.

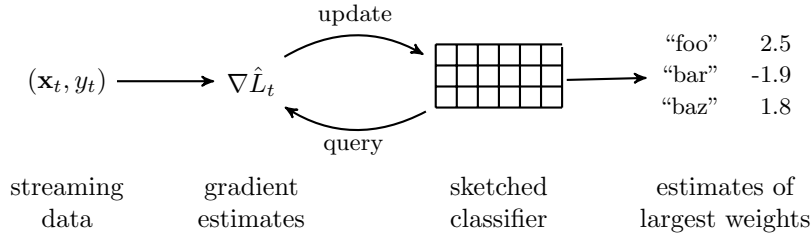
\begin{figure}
  \centering
  \begin{tikzpicture}[shorten >=1pt,>=stealth',thick,scale=1.4]

\node (xy) at (0, 1.05) {$(\mathbf{x}_t, y_t)$};
\node (nabla) at (1.75, 1.05) {$\nabla \hat{L}_t$};
\node (gridtop) at (3.5, 1.25) {};
\node (gridbot) at (3.5, 0.8) {};
\node (gridright) at (4.65, 1.06) {};
\node (feats) at (6.25, 1.1) {\begin{tabular}{@{}rr@{}} \small{``foo''} & \small{2.5}  \\ \small{``bar''} & \small{-1.9} \\ \small{``baz''} & \small{1.8} \end{tabular}};

\draw [step=0.2, shift={(3.5, 0.8)}] (0, 0) grid (1.2, 0.6);

\node (data) at (0, -0.12) {\begin{tabular}{c} streaming \\ data \end{tabular}};
\node (grad) at (1.75, -0.12) {\begin{tabular}{c}  gradient \\ estimates \end{tabular}};
\node (z) at (4.1, -0.12) {\begin{tabular}{c} sketched \\ classifier \end{tabular}};
\node (approx) at (6.25, -0.12) {\begin{tabular}{c}estimates of \\ largest weights \end{tabular}};

\draw [->] (xy) -- (nabla);
\draw [->] (nabla) to [bend left=30] node[above]{\small update} (gridtop);
\draw [->] (gridbot) to [bend left=30] node[below]{\small query} (nabla);
\draw [->] (gridright) to (feats);


\end{tikzpicture}
  \caption{Overview of our approach, where online updates are applied to a sketched (i.e., compressed) classifier from which estimates of the largest weights can be retrieved.}
  \label{fig:overview}
\end{figure}

Online learning in memory-constrained environments is particularly challenging in high-dimensional feature spaces.
For example, consider a spam classifier on text data that is continually updated as new messages are observed and labeled as \emph{spam} or \emph{not spam}. The memory cost of retaining n-gram features grows rapidly as new token combinations are observed in the stream. In an experiment involving an $\sim$80M token newswire dataset \citep{chelba2013one}, we recorded $\sim$47M unique word pairs that co-occur within 5-word spans of text. Disregarding the space required to store strings, maintaining integer vocabulary indexes and 32-bit floating point weights for each of these features would require approximately 560MB of memory. Thus, the memory footprint of classifiers over high-dimensional streaming data can quickly exceed the memory constraints of many deployment environments. Moreover, it is not sufficient to simply apply existing sketches for identifying frequently-occurring features, since the features that occur most often are not necessarily the most discriminative.

In this work, we develop a new sketching algorithm that targets ML applications in these memory-constrained settings.
Building on prior work on sketching for identifying frequent items, we introduce the \emph{Weight-Median Sketch} (WM-Sketch) 
for learning compressed linear classifiers over data streams.
Figure~\ref{fig:overview} illustrates the high-level approach: we first allocate a fixed region of memory as the sketch data structure, and as new examples are observed in the stream, the weights stored in this structure are updated via gradient descent on a given loss function.
In contrast to previous work that employs the ``hashing trick'' to reduce the memory footprint
of a classifier \citep{shi2009hash,weinberger2009feature}, the WM-Sketch supports the approximate recovery of the
most heavily-weighted features in the classifier: at any time, we can efficiently return a list of the top-$K$ features
along with estimates of their weights in an uncompressed classifier trained over the same sequence of examples.

The ability to retrieve heavily-weighted features from the WM-Sketch confers several benefits.
First, the sketch provides a classifier with low memory footprint that retains a
degree of \emph{model interpretability}. This is often practically important as understanding which features are most influential in making predictions
 is relevant to feature selection~\citep{zhang2016materialization},
model debugging, issues of fairness in ML systems~\citep{corbett2017algorithmic}, and human perceptions of model trustworthiness~\citep{ribeiro2016should}.
Second, the ability to retrieve heavily-weighted features enables the execution of a range of analytics workloads that can
be formulated as classification problems over streaming data. In this paper, we demonstrate the effectiveness
of the WM-Sketch in three such applications: (i) streaming data explanation \citep{bailis2017macrobase,meliou-tutorial},
(ii) detecting large relative differences between data streams (i.e., detecting relative deltoids)~\citep{cormode2005s}
and (iii) streaming identification of highly-correlated pairs of features via pointwise mutual information \citep{durme2009streaming}.
The WM-Sketch is able to perform these analyses while using far less memory than uncompressed classifiers.

The key intuition behind the WM-Sketch is that by randomly projecting the gradient updates to a linear
classifier, we can incrementally maintain a compressed version of the
true, high-dimensional model. By choosing this random projection appropriately, we can support
efficient approximate recovery of the model weights.
In particular, the WM-Sketch maintains a Count-Sketch projection~\citep{charikar2002finding}
of the weight vector of the linear classifier.
However, unlike Heavy Hitters sketches that simply increment or decrement counters, the
WM-Sketch updates its state using online gradient descent~\citep{hazan2016introduction}.
Since these updates themselves depend on the current weight estimates, a careful
analysis is needed to ensure that the estimated weights do not diverge from the
true (uncompressed) model parameters over the course of multiple online updates.

We analyze the WM-Sketch both theoretically and empirically. Theoretically, we provide guarantees on the
approximation error of these weight estimates, showing that it is possible to accurately recover large-magnitude
weights using space sub-linear in the feature dimension.
We describe an optimized variant, the Active-Set Weight-Median Sketch (AWM-Sketch) that
outperforms alternative memory-constrained algorithms in experiments on benchmark datasets. For example,
on the standard Reuters RCV1 binary classification benchmark, the AWM-Sketch 
recovers the most heavily-weighted features in the model with 4$\times$ better
approximation error than a frequent-features baseline using the Space Saving algorithm \citep{metwally2005efficient}
and 10$\times$ better than a na\"{i}ve weight truncation baseline, while using the same
amount of memory. Moreover, we demonstrate that the additional interpretability of the AWM-Sketch
does not come at the cost of reduced classification accuracy: empirically,
the AWM-Sketch in fact improves on the classification accuracy of feature hashing, which does not support
weight recovery.

To summarize, we make the following contributions in this work:
\begin{itemize}
  \item We introduce the Weight-Median Sketch, a new sketch for learning linear classifiers over data streams that supports approximate retrieval of the most heavily-weighted features.
  \item We provide a theoretical analysis that provides guarantees on the accuracy of the WM-Sketch estimates. In particular, we show that for feature dimension $d$ and with success probability $1 - \delta$, we can learn a compressed model of dimension $O\left(\epsilon^{-4}\log^3(d/\delta)\right)$ that supports approximate recovery of the optimal weight vector $\mathbf{w}_*$, where the absolute error of each weight is bounded above by $\epsilon \|\mathbf{w}_*\|_1$.
  \item We empirically demonstrate that the optimized AWM-Sketch outperforms several alternative methods in terms of memory-accuracy trade-offs across a range of real-world datasets.
\end{itemize}

Our implementations of the WM-Sketch, AWM-Sketch and the baselines evaluated in our experiments are available at \url{https://github.com/stanford-futuredata/wmsketch}.
\section{Related Work}
\label{sec:related-work}

\vspace{-2pt}
\minihead{Heavy Hitters in Data Streams} Given a sequence of items, the heavy hitters problem is to return the set of all
items whose frequency exceeds a specified fraction of the total number of items.
Algorithms for finding frequent items include counter-based approaches \citep{manku2002approximate,demaine2002frequency,karp2003simple,metwally2005efficient}, quantile algorithms 
\citep{greenwald2001space,shrivastava2004medians}, and sketch-based methods \citep{charikar2002finding,cormode2005improved}. Mirylenka et al. \citep{mirylenka2015conditional} develop streaming algorithms for finding \emph{conditional} heavy hitters, i.e. items that are frequent
in the context of a separate ``parent'' item. Our proposed sketch builds on the Count-Sketch \citep{charikar2002finding}, which was originally 
introduced for identifying frequent items.
In Sec.~\ref{sec:problem-statement}, we show how frequency estimation can in fact be related to the problem of estimating classifier weights.

\vspace{-1pt}
\minihead{Characterizing Changes in Data Streams} \citet{cormode2005s} propose a Count-Min-based algorithm for identifying items whose frequencies change significantly, while \citet{schweller2004reversible} propose the use of reversible hashes to avoid storing key information. In order to explain anomalous traffic flows, \citet{brauckhoff2012anomaly} use histogram-based detectors and association rules to detect large absolute differences. In our network monitoring application (Sec.~\ref{sec:applications}), we focus instead on detecting large \emph{relative} differences, a problem which has previously been found to be challenging \citep{cormode2005s}.

\vspace{-1pt}
\minihead{Resource-Constrained and On-Device Learning} In contrast to \emph{federated learning}, where the goal is to learn a global model on distributed data \citep{konevcny2015federated} or to enforce global regularization on a collection of local models \citep{smith2017federated}, our focus is on memory-constrained learning on a single device without communication over a network. \citet{gupta2017protonn} and \citet{kumar2017resource} perform inference with small-space classifiers on IoT devices, whereas we focus on online learning. Unlike budget kernel methods that aim the reduce the number of stored examplars \citep{crammer2004online,dekel2006forgetron}, our methods instead reduce the dimensionality of feature vectors. Our work also differs from \emph{model compression} or \emph{distillation} \citep{bucilu?2006model,ba2014deep,hinton2015distilling}, which aims to imitate a large, expensive model using a smaller one with lower memory and computation costs---in our setting, the full uncompressed model is never instantiated and the compressed model is learned directly.

\vspace{-1pt}
\minihead{Sparsity-Inducing Regularization} $\ell_1$-regularization is a standard technique for encouraging parameter sparsity in online learning \citep{langford2009sparse,duchi2009efficient,xiao2010dual}. In practice, it is difficult to \emph{a priori} select an $\ell_1$-regularization strength in order to satisfy a given sparsity budget. Here, we propose a different approach: we first fix a memory budget and then use the allocated space to approximate a classifier, with the property that our approximation will be better for sparse parameter vectors with small $\ell_1$-norm.

\vspace{-1pt}
\minihead{Learning Compressed Classifiers} Feature hashing \citep{shi2009hash,weinberger2009feature} is a technique where the classifier is trained on features that have been hashed to a fixed-size table. This approach lowers memory usage by reducing the dimension of the feature space, but at the cost of model interpretability. Our sketch is closely related to this approach---we show that an appropriate choice of random projection enables the recovery of model weights. \citet{calderbankcompressed} describe \emph{compressed learning}, where a classifier is trained on compressively-measured data. The authors focus on classification performance in the compressed domain and do not consider the problem of recovering weights in the original space.

\section{Background}
\label{sec:preliminaries}

In Section~\ref{sec:background}, we review the relevant material on random projections for dimensionality reduction.
In Section~\ref{sec:online-learning}, we describe \emph{online learning}, which models learning on streaming data.

\minihead{Conventions and Notation} The notation $w_i$ denotes the $i$th element of the vector $\mathbf{w}$. The notation $[n]$ denotes the set $\{1, \dots, n \}$. We write $p$-norms as $\|\mathbf{w}\|_p$, where the $p$-norm of $\mathbf{w}\in\mathbb{R}^d$ is defined as $\|\mathbf{w}\|_p \coloneqq \left(\sum_{i=1}^d |w_i|^p\right)^{1/p}$. The infinity-norm $\|\mathbf{w}\|_\infty$ is defined as $\|\mathbf{w}\|_\infty \coloneqq \max_i |w_i|$.

\subsection{Dimensionality Reduction via Random~Projection}
\label{sec:background}

\minihead{Count-Sketch} The Count-Sketch \citep{charikar2002finding} is a linear projection of a vector $\mathbf{x}\in \mathbb{R}^d$ that supports efficient approximate recovery of the entries of $\mathbf{x}$.
The sketch of $\mathbf{x}$ can be built incrementally as entries are observed in a stream---for example, $\mathbf{x}$ can be a vector of counts that is updated as new items are observed.

For a given size $k$ and depth $s$, the Count-Sketch algorithm maintains a collection of $s$ hash tables, each with width $k/s$ (Figure~\ref{fig:cs}). Each index $i \in [d]$ is assigned a random bucket $h_j(i)$ in table $j$ along with a random sign $\sigma_j(i)$. Increments to the $i$th entry are multiplied by $\sigma_j(i)$ and then added to the corresponding buckets $h_j(i)$. The estimator for the $i$th coordinate is the median of the values in the assigned buckets multiplied by the corresponding sign flips. \citet{charikar2002finding} showed the following recovery guarantee for this procedure:

\vspace{-0.5ex}
\begin{lemma}\label{lem:count-sketch}
	\citep{charikar2002finding} Let $\mathbf{x}_\mathrm{cs}$ be the Count-Sketch estimate of the vector $\mathbf{x}$. For any vector $\mathbf{x}$, with probability $1-\delta$, a Count-Sketch matrix with width $\Theta(1/\epsilon^2)$ and depth $\Theta(\log(d/\delta))$ satisfies
	\vspace{-0.25ex}
	\[
	\lpnorm{\mathbf{x}-\mathbf{x}_\mathrm{cs}}{\infty} \le \epsilon \norm{\mathbf{x}}.
	\]
\end{lemma}
In other words, point estimates of each entry of the vector $\mathbf{x}$ can be computed from its compressed form $\mathbf{x}_\mathrm{cs}$. This enables accurate recovery of high-magnitude entries that comprise a large fraction of the norm $\|\mathbf{x}\|_2$.

	\begin{figure}
		\centering
		\centering
		\includegraphics[height=1.3in]{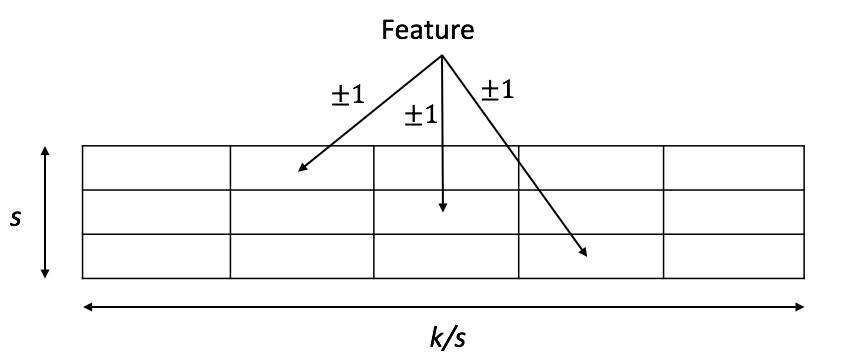}
		\caption{An illustration of the Count-Sketch of size $k$ with depth $s$ and width $k/s$. Each feature hashes to $s$ locations, multiplied by a random $\pm1$ sign.}
		\label{fig:cs}
	\end{figure}

\minihead{Johnson-Lindenstrauss (JL) property} A random projection matrix is said to have the Johnson-Lindenstrauss (JL) property \citep{johnson1984extensions} if it preserves the norm of a vector with high probability:
\begin{definition}\label{def:JL}
	A random matrix $R\in \mathbb{R}^{k \times d}$ has the JL property with error $\epsilon$ and failure probability $\delta$ if for any given $\mathbf{x}\in \mathbb{R}^d$, we have with probability $1-\delta$:
	\[
	\Big| \|R\mathbf{x}\|_2 - \|\mathbf{x}\|_2 \Big| \le \epsilon\| \mathbf{x}\|_2.
	\]
\end{definition}
The JL property holds for dense matrices with independent Gaussian or Bernoulli entries \citep{achlioptas2003database}, and recent work has shown that it applies to certain sparse matrices as well \citep{kane2014sparser}. Intuitively, JL matrices preserve the geometry of a set of points, and we leverage this key fact to ensure that we can still recover the original solution after projecting to low dimension.

\subsection{Online Learning}
\label{sec:online-learning}

The online learning framework deals with learning on a stream of examples, where the model is updated over a series of rounds $t=1,2,\dots,T$.
In each round, we update the model weights $\mathbf{w}_t$ via the following process: (1) receive an input example $(\mathbf{x}_t, y_t)$, (2) incur loss $L_t(\mathbf{w}_t) \coloneqq\ell(\mathbf{w}_t, \mathbf{x}_t, y_t)$, where $\ell$ is a given loss function, and (3) update weights $\mathbf{w}_t$ to $\mathbf{w}_{t+1}$.
There are numerous algorithms for updating the model weights (e.g., \citep{hazan2007logarithmic,duchi2009efficient,xiao2010dual}). In our algorithm, we use online gradient descent (OGD) [\citealp{hazan2016introduction}; Chp. 3], which uses the following update rule:
\begin{equation}
  \mathbf{w}_{t+1} = \mathbf{w}_t - \eta_t \nabla L_t(\mathbf{w}_t)\nonumber,
\end{equation}
where $\eta_t > 0$ is the learning rate at step $t$. OGD enjoys the advantages of simplicity and minimal space requirements: we only need to maintain a representation of the weight vector $\mathbf{w}_t$ and a global scalar learning rate.

\section{Problem Statement}
\label{sec:problem-statement}

We focus on online learning for binary classification with linear models. We observe a stream of examples $(\mathbf{x}_t, y_t)$, where each $\mathbf{x}_t \in \mathbb{R}^d$ is a feature vector and each $y_t \in\{-1, +1 \}$ is a binary label. A linear classifier parameterized by weights $\mathbf{w}\in\mathbb{R}^d$ makes predictions $\hat{y}$ by returning $+1$ for all inputs with non-negative inner product with $\mathbf{w}$, and $-1$ otherwise: $\hat{y} = \mathrm{sign}\left(\mathbf{w}^T \mathbf{x}\right)$. The goal of learning is to select weights $\mathbf{w}$ that minimize the total loss $\sum_t L_t(\mathbf{w})$ on the observed data. In the following, we refer to $\mathbf{w}$ interchangeably as the \emph{weights} and as the \emph{classifier}.

Suppose we have observed $T$ examples in the stream, and consider the classifier $\mathbf{w}_*$ that minimizes the loss over those $T$ examples. It may not be possible to precisely represent each entry\footnote{For example, storing each nonzero entry as a single-precision floating point number.} of the vector $\mathbf{w}_*$ within a memory budget that is much less than the cost of representing a general vector in $\mathbb{R}^d$. In particular, $\mathbf{w}_*$ may be a dense vector. Thus, it may not be possible to represent $\mathbf{w}_*$ in a memory-constrained setting, and in practical applications this is particularly problematic when the dimension $d$ is large.

For a fixed memory budget $B$, our goal is to obtain a summary $\mathbf{z}$ that uses space at most $B$ from which we are able to estimate the value of each entry of the optimal classifier $\mathbf{w}_*$. We formalize this problem as the Weight Estimation Problem, which we make precise in the following section. In addition to supporting weight estimation, a secondary goal is to be able to use the summary $\mathbf{z}$ to perform classification on data points $\mathbf{x}$ via some inference function $f$, i.e. $\hat{y} = f(\mathbf{z}, \mathbf{x})$. We would like to classify data points using the summary without too much additional error compared to $\mathbf{w}_*$.

\subsection{The Weight Estimation Problem}
\label{sec:problem}

In this section, we formalize the problem of estimating the weights of the optimal classifier $\mathbf{w}_*$ from a compact summary.
To facilitate the presentation of this problem and to build intuition, we highlight the connection between our goal of weight estimation and
previous work on the approximate recovery of frequency estimates from compressed count vectors. To this end, we formalize a
general problem setup that subsumes both the approximate recovery of frequencies and the approximate recovery of weights in linear classifiers
as special cases.

The $\epsilon$-approximate frequency estimation problem can be defined as follows:
\begin{definition}{\citep{cormode2008finding} ($\epsilon$-Approximate Frequency Estimation)}\label{def:freq-estimation-problem}
Given a sequence of $T$ items, each drawn from the set $[d]$, let $v_i$ denote the count of the number of times item $i$ is seen over the stream. The $\epsilon$-approximate frequency estimation problem is to return, for any $i \in [d]$, a value $\hat{v}_i$ such that $|\hat{v}_i - v_i| \leq \epsilon T$.
\end{definition}
The frequency estimation problem commonly appears in the context of finding heavy hitters---i.e., items whose frequencies exceed a given threshold $\phi T$. Given an algorithm that solves the $\epsilon$-approximate frequency estimation problem, we can then find all heavy hitters (possibly with false positives) by returning all items with estimated frequency above $(\phi - \epsilon)T$.

We now define an analogous setup for online convex optimization problems that formalizes our goal of weight recovery from summarized classifiers:
\begin{definition}{(($\epsilon, p$)-Approximate Weight Estimation for Convex Functions)}
Given a sequence of $T$ convex functions $L_t : \mathcal{X} \rightarrow \mathbb{R}$ over a convex domain $\mathcal{X} \subseteq \mathbb{R}^d$, let $\mathbf{w}_* \coloneqq \argmin_\mathbf{w} \sum_{t=1}^T L_t(\mathbf{w})$. The $(\epsilon, p)$-approximate weight estimation problem is to return, for any $i\in[d]$, a value $\hat{w}_i$ such that $|\hat{w}_i - (w_*)_i| \leq \epsilon\|\mathbf{w}_*\|_p.$
\end{definition}

Note that frequency estimation~(Definition~\ref{def:freq-estimation-problem}) can be viewed as a special case of this problem.
Set $L_t(\mathbf{w}) = -\mathbf{w}^T \mathbf{x}_t$,
where $(x_t)_i = 1$ if item $i$ is observed at time $t$ and $0$ otherwise (assume that only one item is observed at each $t$), define $\mathbf{x}_{1:T} \coloneqq \sum_{t=1}^T \mathbf{x}_t$, and let $\mathcal{X} = \{ \mathbf{w} : \|\mathbf{w}\|_2 \leq \|\mathbf{x}_{1:T}\|_2 \}$. 
Then $\mathbf{w}_* = \mathbf{x}_{1:T}$, and we note that $\|\mathbf{w}_*\|_1 = T$. Thus, the frequency estimation problem is an instance of the $(\epsilon, 1)$-approximate weight estimation problem.

\minihead{Weight Estimation for Linear Classifiers} We now specialize to the case of online learning for linear classifiers. Define the losses $L_t$ as:
\begin{equation}
\label{eq:objective}
  L_t(\mathbf{w}) = \ell\left(y_t\mathbf{w}^T \mathbf{x}_t\right) + \frac{\lambda}{2} \|\mathbf{w}\|_2^2,
\end{equation}
where $\ell$ is a convex, differentiable function, $(\mathbf{x}_t, y_t)$ is the example observed at time $t$, and $\lambda > 0$ controls the strength of $\ell_2$-regularization.
The choice of $\ell$ defines the linear classification model to be used.
For example, the logistic loss $\ell(\tau) = \log(1 + \exp(-\tau))$ defines logistic regression, and smoothed versions of the hinge loss $\ell(\tau) = \max\{0, 1 - \tau\}$ define close relatives of linear support vector machines.

To summarize, for each time step, we wish to maintain a \emph{compact summary} $\mathbf{z}_t$ that allows us to estimate each weight in the optimal classifier $\mathbf{w}_*$ over all the examples seen so far in the stream. In the following sections, we describe a method for maintaining such a summary and provide theoretical guarantees on the accuracy of the recovered weights.

\section{Finding Heavily-Weighted Features}
\label{sec:method}

In this section, we describe our proposed method, the Weight-Median Sketch (WM-Sketch), along with a simple variant, the Active-Set Weight-Median Sketch (AWM-Sketch), that empirically improves on the basic WM-Sketch in both classification and recovery accuracy.

\vspace{-1ex}
\subsection{Weight-Median Sketch}
\label{sec:wm-sketch}

The main data structure in the WM-Sketch is identical to that used in the Count-Sketch. The sketch is parameterized by size $k$, depth $s$, and width $k/s$. We initialize the sketch with a size-$k$ array set to zero.
For a given depth $s$, we view this array as being arranged in $s$ rows, each of width $k/s$ (assume that $k$ is a multiple of $s$). 
We denote this array as $\mathbf{z}$, and equivalently view it as a vector in $\mathbb{R}^k$.

The high-level idea is that each row of the sketch is a compressed version of the model weight vector $\mathbf{w}\in\mathbb{R}^d$, where each index $i \in [d]$ is mapped to some assigned bucket $j \in [k/s]$. Since $k/s \ll d$, there will be many collisions between these weights; therefore, we maintain $s$ rows---each with different assignments of features to buckets---in order to disambiguate weights.

\minihead{Hashing Features to Buckets} In order to avoid explicitly storing the mapping from features to buckets, which would require space linear in $d$, we implement the mapping using hash functions as in the Count-Sketch. For each row $j \in [s]$, we maintain a pair of hash functions, $h_j : [d] \rightarrow [k/s]$ and $\sigma_j : [d] \rightarrow \{-1, +1\}$. Let the matrix $A \in \{-1, +1\}^{k\times d}$ denote the Count-Sketch projection implicitly represented by the hash functions $h_j$ and $\sigma_j$, and let $R$ be a scaled version of this projection, $R = \frac{1}{\sqrt{s}}A$.
We use the projection $R$ to compress feature vectors and update the sketch.

\minihead{Updates} We update the sketch by performing gradient descent updates directly on the compressed classifier $\mathbf{z}$.
We compute gradients on a ``compressed'' version $\hat{L}_t$ of the regularized loss $L_t$ defined in Eq.~\ref{eq:objective}:
\begin{equation}\nonumber
  \hat{L}_t(\mathbf{z}) = \ell\left(y_t \mathbf{z}^T R\mathbf{x}_t \right) + \frac{\lambda}{2} \|\mathbf{z}\|_2^2.
\end{equation}
This yields the following update to $\mathbf{z}$:
\begin{equation}\nonumber
  \hat{\Delta}_t \coloneqq -\eta_t \nabla \hat{L}_t(\mathbf{z}) = -\eta_t \left(y_t \nabla\ell(y_t \mathbf{z}^T R \mathbf{x}_t) R\mathbf{x}_t + \lambda \mathbf{z} \right).
\end{equation}

To build intuition, it is helpful to compare this update to the Count-Sketch update rule \citep{charikar2002finding}. In the frequent items setting, the input $\mathbf{x}_t$ is a one-hot encoding for the item seen in that time step. The update to the Count-Sketch state $\mathbf{z}_\mathrm{cs}$ is the following:
\begin{equation}
\label{eq:cs-update}
  \tilde{\mathbf{\Delta}}^\mathrm{cs}_t = {A}\mathbf{x}_t,\nonumber
\end{equation}
where $A$ is defined identically as above. Ignoring the regularization term, our update rule is simply the Count-Sketch update scaled by the constant $-\eta_t y_t s^{-1/2} \nabla\ell(y_t\mathbf{z}^T R\mathbf{x}_t)$. However, an important detail to note is that the Count-Sketch update is \emph{independent} of the sketch state $\mathbf{z}_\mathrm{cs}$, whereas the WM-Sketch update does depend on $\mathbf{z}$. This cyclical dependency between the state and state updates is the main challenge in our analysis of the WM-Sketch.

\begin{algorithm}[t]
  \DontPrintSemicolon
  \SetAlgoLined
  \SetAlgoNoEnd
  \SetKwFunction{FDot}{Dot}
  \SetKwFunction{FUpdate}{Update}
  
  \newcommand\mycommfont[1]{\rmfamily{#1}}
  \SetCommentSty{mycommfont}
  \SetKwComment{Comment}{$\triangleright$ }{}
  
  \Input{size~$k$, depth~$s$, loss function~$\ell$, $\ell_2$-regularization parameter~$\lambda$, learning rate schedule~$\eta_t$}
  
  \Init{}{
	  $\mathbf{z} \leftarrow s\times k / s$ array of zeroes \;
	  Sample $R$, a Count-Sketch matrix scaled by $\frac{1}{\sqrt{s}}$\;
	  $t \leftarrow 0$
  }
  
  \Fn{\FUpdate{$\mathbf{x}$, $y$}}{
    $\tau \leftarrow \mathbf{z}^T R \mathbf{x}\quad$ \Comment*[f]{Prediction for $\mathbf{x}$}
    
    $\mathbf{z} \leftarrow (1 - \lambda\eta_t)\mathbf{z} - \eta_t y \nabla\ell\left(y\tau\right) R\mathbf{x} $\;
    $t \leftarrow t + 1$\;
  }
  
  \SetKwFunction{FQuery}{Query}
  \Fn{\FQuery{$i$}}{
    \textbf{return} output of Count-Sketch retrieval on $\sqrt{s}\mathbf{z}$
  }
  
  \caption{Weight-Median (WM) Sketch}
  \label{alg:hwsketch}
\end{algorithm}

\minihead{Queries} To obtain an estimate $\hat{w}_i$ of the $i$th weight, we return the median of the values $\{\sqrt{s} \sigma_j(i) z_{j,h_j(i)} :\; j\in[s]\}$. Save for the $\sqrt{s}$ factor, this is identical to the query procedure for the Count-Sketch.

\smallskip

We summarize the update and query procedures for the WM-Sketch in Algorithm~\ref{alg:hwsketch}. In the next section, we show how the sketch size $k$ and depth $s$ parameters can be chosen to satisfy an $\epsilon$ approximation guarantee with failure probability $\delta$ over the randomness in the sketch matrix.

\minihead{Efficient Regularization} A na\"{i}ve implementation of $\ell_2$ regularization on $\mathbf{z}$ that scales each entry in $\mathbf{z}$ by $(1 - \eta_t \lambda)$ in each iteration incurs an update cost of $O(k + s \cdot \mathrm{nnz}(\mathbf{x}))$. This masks the computational gains that can be realized when $\mathbf{x}$ is sparse. Here, we use a standard trick \citep{shalev2011pegasos}: we maintain a global \emph{scale} parameter $\alpha$ that scales the sketch values $\mathbf{z}$. Initially, $\alpha = 1$ and we update $\alpha \leftarrow (1 - \eta_t \lambda) \alpha$ to implement weight decay over the entire feature vector. Our weight estimates are therefore additionally scaled by $\alpha$: $\hat{w}_i = \mathrm{median}\left\{ \sqrt{s} \alpha \sigma_j(i) z_{j, h_j(i)} :\; j \in [s] \right\}$. This optimization reduces the update cost from $O(k + s \cdot \mathrm{nnz}(\mathbf{x}))$ to $O(s \cdot \mathrm{nnz}(\mathbf{x}))$.

\subsection{Active-Set Weight-Median Sketch}

We now describe a simple, heuristic extension to the WM-Sketch that significantly improves the recovery accuracy
of the sketch in practice. We refer to this variant as the Active-Set Weight-Median Sketch (AWM-Sketch).

To efficiently track the top elements across
sketch updates, we can use a min-heap ordered by the absolute value of the estimated weights.
This technique is also used alongside heavy-hitters sketches to identify the most frequent items in the stream \citep{charikar2002finding}.
In the basic WM-Sketch, the heap merely functions as a mechanism to passively maintain
the heaviest weights. This baseline scheme can be improved by noting that weights that are already stored
in the heap need not be tracked in the sketch; instead, the sketch can be updated lazily only when
the weight is evicted from the heap. 
This heuristic has previously been used in the context of improving count estimates derived from
a Count-Min Sketch \citep{roy2016augmented}.
The intuition here is the following: since we are already maintaining a heap of heavy
items, we can utilize this structure to reduce
error in the sketch as a result of collisions with heavy items.

The heap can be thought of as an ``active set'' of high-magnitude weights, while the sketch estimates the contribution of the tail of the weight vector. Since the weights in the heap are represented exactly, this active set heuristic should intuitively yield better estimates of the heavily-weighted features in the model.

\begin{algorithm}[t]
  \DontPrintSemicolon
  \SetAlgoLined
  \SetAlgoNoEnd
  
  \newcommand\mycommfont[1]{\rmfamily{#1}}
  \SetCommentSty{mycommfont}
  \SetKwComment{Comment}{$\triangleright$ }{}
  
  \Init{}{
    $S \leftarrow \{\} \quad$ \Comment*[f]{Empty heap}
    
    $\mathbf{z} \leftarrow s\times k / s$ array of zeroes \;
	Sample $R$, a Count-Sketch matrix scaled by $\frac{1}{\sqrt{s}}$\;
	$t \leftarrow 0$
  }
  
  \Fn{\FUpdate{$\mathbf{x}$, $y$}}{
    $\mathbf{x}_s \leftarrow \{ x_i : i \in S \}\quad$\Comment*[f]{Features in heap}
    
    $\mathbf{x}_\mathrm{wm} \leftarrow \{ x_i : i \not\in S \}\quad$\Comment*[f]{Features in sketch}
    
    $\tau \leftarrow \sum_{i\in S} S[i]\cdot x_i + \mathbf{z}^T R \mathbf{x}_\mathrm{wm}$ \Comment*[f]{Prediction for $\mathbf{x}$}
    
	$S \leftarrow (1 - \lambda\eta_t) S - \eta_t y \nabla\ell(y\tau) \mathbf{x}_s$ \Comment*[f]{Heap update}
	
	$\mathbf{z} \leftarrow (1 - \lambda\eta_t) \mathbf{z}\quad$ \Comment*[f]{Apply regularization}

    \For{$i \not\in S$}{
    \Comment{Either update $i$ in sketch or move to heap}
      $\tilde{w} \leftarrow \FQuery{i} - \eta_t yx_i\nabla\ell(y\tau)$\;
      $i_\mathrm{min} \leftarrow \argmin_j(|S[j]|)$\;
      \eIf{$|\tilde{w}| > |S[i_\mathrm{min}]|$}{    
        Remove $i_\mathrm{min}$ from $S$\;
        Add $i$ to $S$ with weight $\tilde{w}$\;
        Update $i_\mathrm{min}$ in sketch with $S[i_\mathrm{min}] - \FQuery($${i_\mathrm{min}})$\;
      }{
        Update $i$ in sketch with $\eta_t yx_i \nabla\ell(y\tau)$\;
      }
    }
    $t \leftarrow t + 1$\;
  }
  \caption{Active-Set Weight-Median (AWM) Sketch}
  \label{alg:hwactive}
\end{algorithm}

As a general note, similar coarse-to-fine approximation schemes have been proposed in other online learning settings. A similar scheme for memory-constrained sparse linear regression was analyzed by \citet{steinhardt2015minimax}. Their algorithm similarly uses a Count-Sketch for approximating weights, but in a different setting ($K$-sparse linear regression) and with a different update policy for the active set.


\vspace{-6pt}
\section{Theoretical Analysis}
\label{sec:analysis}

We derive bounds on the recovery error achieved by the WM-Sketch for given settings of the size $k$ and depth $s$.
The main challenge in our analysis is that the updates to the sketch depend on gradient estimates which in turn depend on the state
of the sketch. This reflexive dependence makes it difficult to straightforwardly transplant the standard analysis for the Count-Sketch to
our setting. Instead, we turn to ideas drawn from norm-preserving random projections and online convex optimization.

In this section, we begin with an analysis of recovery error in the batch setting, where we are given access to a fixed dataset of size $T$ consisting of the first $T$ examples observed in the stream and are allowed multiple passes over the data. Subsequently, we use this result to show guarantees in a restricted online case where we are only allowed a single pass through the data, but with the assumption that the order of the data is not chosen adversarially.

\subsection{Batch Setting}

First, we briefly outline the main ideas in our analysis. With high probability, we can sample a random projection to dimension $k \ll d$ that satisfies the JL norm preservation property (Definition~\ref{def:JL}). We use this property to show that for any fixed dataset of size $T$, optimizing a projected version of the objective yields a solution that is close to the projection of the minimizer of the original, high-dimensional objective. Since our specific construction of the JL projection is also a Count-Sketch projection, we can make use of existing error bounds for Count-Sketch estimates to bound the error of our recovered weight estimates.

Let $R\in\mathbb{R}^{k\times d}$ denote the scaled Count-Sketch matrix defined in Sec.~\ref{sec:wm-sketch}.
This is the hashing-based sparse JL projection proposed by Kane and Nelson \citep{kane2014sparser}. We consider the following pair of objectives defined over the $T$ observed examples---the first defines the problem in the original space and the second defines the corresponding problem where the learner observes sketched examples $(R\mathbf{x}_t, y_t)$:
\begin{align*}
  L(\mathbf{w}) &= \frac{1}{T}\sum_{t=1}^{T}\ell\left(y_t \mathbf{w}^T \mathbf{x}_t\right) + \frac{\lambda}{2}\|\mathbf{w}\|_2^2, \\
  \hat{L}(\mathbf{z}) &= \frac{1}{T}\sum_{t=1}^{T}\ell\left(y_t \mathbf{z}^T R\mathbf{x}_t\right) + \frac{\lambda}{2}\|\mathbf{z}\|_2^2.
\end{align*}

Suppose we optimized these objectives over $\mathbf{w}\in\mathbb{R}^d$ and $\mathbf{z}\in\mathbb{R}^k$ respectively to obtain solutions $\mathbf{w}_* = \argmin_\mathbf{w} L(\mathbf{w})$ and $\mathbf{z}_* = \argmin_{\mathbf{z}}\hat{L}(\mathbf{z})$. How then does $\mathbf{w}_*$ relate to $\mathbf{z}_*$ given our choice of sketching matrix $R$ and regularization parameter $\lambda$? Intuitively, if we stored all the data observed up to time $T$ and optimized $\mathbf{z}$ over this dataset, we should hope that the optimal solution $\mathbf{z}_*$ is close to $R\mathbf{w}_*$, the sketch of $\mathbf{w}_*$, in order to have any chance of recovering the largest weights of $\mathbf{w}_*$. We show that in this batch setting, $\norm{\b{z}_*-R\b{w}_*}$ is indeed small; we then use this property to show element-wise error guarantees for the Count-Sketch recovery process.
We now state our result for the batch setting:

\begin{theorem}
	\label{thm:batch-recovery}
	Let the loss function $\ell$ be $\beta$-strongly smooth\footnote{A function $f:\mathcal{X} \rightarrow \mathbb{R}$ is $\beta$-strongly smooth w.r.t. a norm $\|\cdot\|$ if $f$ is everywhere differentiable and if for all $\mathbf{x}, \mathbf{y}$ we have: \begin{align*}f(\mathbf{y}) \leq f(\mathbf{x}) + (\mathbf{y} - \mathbf{x})^T\nabla f(\mathbf{x}) + \frac{\beta}{2}\|\mathbf{y} - \mathbf{x}\|^2.\end{align*}} (w.r.t. $\|\cdot\|_2)$ and $\max_t\lonenorm{\mathbf{x}_t} = \gamma$. For fixed constants $C_1,C_2>0$, let:
	\begin{align*}
	k &= {\left(C_1/\epsilon^4\right)\log^3(d/\delta)\max\left\{1,\beta^2\gamma^4/\lambda^2\right\}},\\
	s &= {\left(C_2/\epsilon^2\right)\log^2(d/\delta)\max\left\{1,\beta \gamma^2/\lambda\right\}}.
	\end{align*}
	Let $\mathbf{w}_*$ be the minimizer of the original objective function $L(\mathbf{w})$ and ${\mathbf{w}}_\mathrm{est}$ be the estimate of $\mathbf{w}_*$ returned by performing Count-Sketch recovery on the minimizer $\mathbf{z}_*$ of the projected objective function $\hat{L}(\mathbf{z})$. Then with probability $1-\delta$ over the choice of $R$,
	\[
	\lpnorm{\mathbf{w}_*-{\mathbf{w}}_\mathrm{est}}{\infty} \le \epsilon \| \mathbf{w}_* \|_1.
	\]
\end{theorem}

We note that for standard loss functions such as the logistic loss and the smoothed hinge loss, we have smoothness parameter $\beta = 1$. Moreover, we can assume that input vectors are normalized so that $\|\mathbf{x}_t\|_1 = 1$, and that typically $\lambda < 1$. Given these parameter choices, we can obtain simpler expressions for the sketch size $k$ and sketch depth $s$:
\begin{align*}
  k &= {O\left(\epsilon^{-4} \lambda^{-2} \log^3(d/\delta)\right)},\\
  s &= {O\left(\epsilon^{-2} \lambda^{-1} \log^2(d/\delta)\right)}.
\end{align*}

We defer the full proof of the theorem to Appendix~\ref{sec:proofs}. We now highlight some salient properties of this recovery result:

\minihead{Sub-linear Dimensionality Dependence} Theorem~\ref{thm:batch-recovery} implies that we can achieve error bounded by $\epsilon \|\mathbf{w}_*\|_1$ with a sketch of size only polylogarithmic in the feature dimension $d$---this implies that memory-efficient learning and recovery is possible in the large-$d$ regime that we are interested in. Importantly, the sketch size $k$ is independent of the number of observed examples $T$---this is crucial since our applications involve learning over data streams of possibly unbounded length.

\minihead{Update Time} Recall that the WM-Sketch can be updated in time $O(s \cdot \mathrm{nnz}(\mathbf{x}))$ for a given input vector $\mathbf{x}$. 
Thus, the sketch supports an update time of $O(\epsilon^{-2} \lambda^{-1} \log^2(d/\delta) \cdot \mathrm{nnz}(\mathbf{x}))$ in each iteration.

\minihead{$\ell_2$-Regularization} $k$ and $s$ scale inversely with the strength of $\ell_2$ regularization: this is intuitive because additional regularization will shrink both $\mathbf{w}_*$ and $\mathbf{z}_*$ towards zero. We observe this inverse relationship between recovery error and $\ell_2$ regularization in practice (see Figure \ref{fig:l2err_reg}).

\minihead{Input Sparsity} The recovery error depends on the maximum $\ell_1$-norm $\gamma$ of the data points $\b{x}_t$, and the bound is most optimistic when $\gamma$ is small. Across all of the applications we consider in Sections \ref{sec:experiments} and \ref{sec:applications}, the data points are sparse with small $\ell_1$-norm, and hence the bound is meaningful across a number of real-world settings. 

\minihead{Weight Sparsity} The per-parameter recovery error in Theorem~\ref{thm:batch-recovery} is bounded above by a multiple of the $\ell_1$-norm of the optimal weights $\mathbf{w}_*$ for the uncompressed problem. This supports the intuition that sparse solutions with small $\ell_1$-norm should be more easily recovered. In practice, we can augment the objective with an additional $\|\mathbf{w}\|_1$ (resp. $\|\mathbf{z}\|_1$) term to induce sparsity; this corresponds to elastic net-style composite $\ell_1$/$\ell_2$ regularization on the parameters of the model~\cite{zou2005regularization}.

\minihead{Comparison with Frequency Estimation} We can compare our guarantees for weight estimation in linear classifiers with existing guarantees for frequency estimation. The Count-Sketch requires $\Theta(\epsilon^{-2} \log(d/\delta))$ space and $\Theta(\log(d/\delta))$ update time to obtain frequency estimates $\mathbf{v}_\mathrm{cs}$ with error $\|\mathbf{v} - \mathbf{v}_\mathrm{cs}\|_\infty \leq \epsilon \|\mathbf{v}\|_2$, where $\mathbf{v}$ is the true frequency vector (Lemma~\ref{lem:count-sketch}). The Count-Min Sketch uses $\Theta\left(\epsilon^{-1} \log(d/\delta)\right)$ space and $\Theta(\log(d/\delta))$ update time to obtain frequency estimates $\mathbf{v}_\mathrm{cm}$ with error $\|\mathbf{v} - \mathbf{v}_\mathrm{cm}\|_\infty \leq \epsilon \|\mathbf{v}\|_1$ \citep{cormode2005improved}. Thus, our analysis yields guarantees of a similar form to bounds for frequency estimation in this more general framework, but with somewhat worse polynomial dependence on $1/\epsilon$ and $\log(d/\delta)$, and additional $1/\epsilon$ dependence in the update time.  

\subsection{Online Setting}

We now provide guarantees for WM-Sketch in the online setting. We make two small modifications to WM-Sketch for the convenience of analysis. First, we assume that the iterate $\b{z}_t$ is projected onto a $\ell_2$ ball of radius $D$ at every step. Second, we also assume that we perform the final Count-Sketch recovery on the average $\bar{\b{z}}=\frac{1}{T}\sum_{i=1}^{T}\b{z}_t$ of the weight vectors, instead of on the current iterate $\b{z}_t$. While using this averaged sketch is useful for the analysis, maintaining a duplicate data structure in practice for the purpose of accumulating the average would double the space cost of our method. Therefore, in our
implementation of the WM-Sketch, we simply maintain the current iterate $\mathbf{z}_t$. As we show in the next section this approach achieves good
performance on real-world datasets, in particular when combined with the active set heuristic.

Our guarantee holds in expectation over uniformly random permutations of $\{(\b{x}_1, y_1), \dots, (\b{x}_T, y_T) \}$. In other words, we achieve low recovery error on average over all orderings in which the $T$ data points could have been presented. We believe this condition is necessary to avoid worst-case adversarial orderings of the data points---since the WM-Sketch update at any time step depends on the state of the sketch itself, adversarial orderings can potentially lead to high error accumulation.

\begin{theorem}
	\label{thm:online}
		Let the loss function $\ell$ be $\beta$-strongly smooth (w.r.t. $\|\cdot\|_2)$, and have its derivative bounded by $H$. Assume $\norm{\b{x}_t}\le 1, \max_t\lonenorm{\mathbf{x}_t}=\gamma$, $\norm{\b{w}_*}\le D_2$ and $\lpnorm{\b{w}_*}{1}\le D_1$. Let $G$ be a bound on the $\ell_2$ norm of the gradient at any time step $t$, in our case $G\le H(1+\epsilon\gamma) + \lambda D$. For fixed constants $C_1,C_2,C_3>0$, let:
		\begin{align*}
		k &= {\left(C_1/\epsilon^4\right)\log^3(d/\delta)\max\left\{1,\beta^2\gamma^4/\lambda^2\right\}},\\
		s &= {\left(C_2/\epsilon^2\right)\log^2(d/\delta)\max\left\{1,\beta \gamma^2/\lambda\right\}},\\
		T&\ge (C_3/\epsilon^4)\zeta\log^2(d/\delta)\max\{1,\beta \gamma^2/\lambda\},
		\end{align*}
where $\zeta = (1 / \lambda^2)(D_2/\lpnorm{w_*}{1})^2(G+(1+\epsilon\gamma)H)^2$. Let $\mathbf{w}_*$ be the minimizer of the original objective function $L(\mathbf{w})$ and ${\mathbf{w}}_\mathrm{wm}$ be the estimate $\mathbf{w}_*$ returned by the WM-Sketch algorithm with averaging and projection on the $\ell_2$ ball with radius $D=(D_2+\epsilon D_1)$. Then with probability $1-\delta$ over the choice of $R$,
		\[
		\E[\lpnorm{\mathbf{w}_*-{\mathbf{w}}_\mathrm{wm}}{\infty}] \le \epsilon \| \mathbf{w}_* \|_1,
		\]
		where the expectation is taken with respect to uniformly sampling a permutation in which the samples are received.
\end{theorem}

Theorem~\ref{thm:online} shows that in this restricted online setting, we achieve a bound with the same scaling of the sketch parameters $k$ and $s$ as the batch setting (Theorem~\ref{thm:batch-recovery}).  Again, we defer the full proof to Appendix~\ref{sec:ol_proof}.

Intuitively, it seems reasonable to expect that we would need an ``average case'' ordering of the stream in order to obtain a similar recovery guarantee to the batch setting. An adversarial, worst-case ordering of the examples could be one where all the negatively-labeled examples are first presented, followed by all the positively-labeled examples. In such a setting, it appears implausible that a single-pass online algorithm should be able to accurately estimate the weights obtained by a batch algorithm that is allowed multiple passes over the data.

\section{Evaluation}
\label{sec:experiments}

In this section, we evaluate the Weight-Median Sketch on three standard binary classification datasets. Our goal here is to compare the WM-Sketch and AWM-Sketch against alternative limited-memory methods in terms of (1) recovery error in the estimated top-$K$ weights, (2) classification error rate, and (3) runtime performance. In the next section, we explore specific applications of the WM-Sketch in stream processing tasks.

\subsection{Datasets and Experimental Setup}

We evaluated our proposed sketches on several standard benchmark datasets as well as in the context of specific streaming applications. Table~\ref{tab:datasets} lists summary statistics for these datasets.

\minihead{Classification Datasets} We evaluate the recovery error on $\ell_2$-regularized online logistic regression trained on three standard binary classification datasets: Reuters RCV1 \citep{lewis2004rcv1}, malicious URL identification \citep{ma2009identifying}, and the \texttt{Algebra} dataset from the KDD Cup 2010 large-scale data mining competition \citep{stamper2010algebra,yu2010feature}. We use the standard training split for each dataset except for the RCV1 dataset, where we use the larger ``test'' split as is common in experimental evaluations using this dataset~\citep{golovin2013large}.

\begin{table}\centering
  \begin{tabular}{@{}lrrr@{}}\toprule
  Dataset & \# Examples & \# Features & Space (MB) \\
  \midrule
  Reuters RCV1 & $6.77\times 10^5$ & $4.72\times 10^4$ & 0.4 \\ 
  Malicious URLs & $2.40\times 10^6$ & $3.23\times 10^6$ & 25.8 \\
  KDD Cup Algebra & $8.41 \times 10^6$ & $2.02 \times 10^7$ & 161.8 \\
  \midrule
  Senate/House Spend. & $4.08\times 10^7$ & $5.14\times 10^5$ & 4.2 \\
  Packet Trace & $1.86\times10^7$ & $1.26\times 10^5$ & 1.0 \\
  Newswire & $2.06\times10^9$ & $4.69\times 10^7$ & 375.2 \\
  \bottomrule
  \end{tabular}
  \caption{Summary of benchmark datasets with the space cost of representing full weight vectors and feature identifiers using 32-bit values. The first set of three consists of standard binary classification datasets used in Sec.~\ref{sec:experiments}; the second set consists of datasets specific to the applications in Sec.~\ref{sec:applications}.}
  \label{tab:datasets}
\end{table}

\begin{table}\centering
  \begin{tabular}{@{}lrrrcrrr@{}}\toprule
   & \multicolumn{3}{c}{WM-Sketch} &\phantom{a}& \multicolumn{3}{c}{AWM-Sketch}  \\
  \cmidrule{2-4} \cmidrule{6-8}
  Budget (KB)  & \multicolumn{1}{c}{$|S|$} & width & depth && \multicolumn{1}{c}{$|S|$} & width & depth \\
   \midrule
  2 & 128 & 128 & 2 && 128 & 256 & 1 \\
  4 & 256 & 256 & 2 && 256 & 512 & 1 \\
  8 & 128 & 128 & 14 && 512 & 1024 & 1 \\
  16 & 128 & 128 & 30 && 1024 & 2048 & 1 \\
  32 & 128 & 256 & 31 && 2048 & 4096 & 1 \\
  \bottomrule
  \end{tabular}
\caption{Sketch configurations with minimal $\ell_2$ recovery error on \texttt{RCV1} dataset ($|S|$ denotes heap capacity).}
\label{tab:l2_configs}
\end{table}

For each dataset, we make a single pass through the set of examples. Across all our experiments, we use an initial learning rate $\eta_0 = 0.1$ and $\lambda \in \{10^{-3}, 10^{-4}, 10^{-5}, 10^{-6}\}$. We used the following set of space constraints: 2KB, 4KB, 8KB, 16KB and 32KB. For each setting of the space budget and for each method, we evaluate a range of configurations compatible with that space constraint; for example, for evaluating the WM-Sketch, this corresponds to varying the space allocated to the heap and the sketch, as well as trading off between the sketch depth $s$ and the width $k/s$. For each setting, we run 10 independent trials with distinct random seeds; our plots show medians and the range between the worst and best run. 

\minihead{Memory Cost Model} In our experiments, we control for memory usage and configure each method to satisfy the given space constraints using the following cost model: we charge 4B of memory utilization for each feature identifier, feature weight, and auxiliary weight (e.g., random keys in Algorithm~\ref{alg:prob-truncation} or counts in the Space Saving baseline) used. For example, a simple truncation instance (Algorithm~\ref{alg:truncation}) with 128 entries uses 128 identifiers and 128 weights, corresponding to a memory cost of 1024B.

\subsection{Recovery Error Comparison}

We measure the accuracy to which our methods are able to recover the top-$K$ weights in the model using the following relative $\ell_2$ error metric:
\begin{equation*}
  \mathrm{RelErr}(\mathbf{w}^K, \mathbf{w}_*) = \frac{\| \mathbf{w}^K - \mathbf{w}_* \|_2}{\| \mathbf{w}_*^K - \mathbf{w}_* \|_2},
\end{equation*}
where $\mathbf{w}^K$ is the $K$-sparse vector representing the top-$K$ weights returned by a given method, $\mathbf{w}_*$ is the weight vector obtained by the uncompressed model, and $\mathbf{w}_*^K$ is the $K$-sparse vector representing the true top-$K$ weights in $\mathbf{w}_*$. The relative error metric is therefore bounded below by $1$ and quantifies the relative suboptimality of the estimated top-$K$ weights.
The best configurations for the WM- and AWM-Sketch on \texttt{RCV1} are listed in Table~\ref{tab:l2_configs}; the optimal configurations for the remaining datasets are similar.

We compare our methods across datasets (Fig.~\ref{fig:l2err_comp}) and across memory constraints on a single dataset (Fig.~\ref{fig:l2err_space}). For clarity, we omit the Count-Min Frequent Features baseline since we found that the Space Saving baseline achieved consistently better performance. We found that the AWM-Sketch consistently achieved lower recovery error than alternative methods on our benchmark datasets. The Space Saving baseline is competitive on \texttt{RCV1} but underperforms the simple Probabilistic Truncation baseline on \texttt{URL}: this demonstrates that tracking frequent features can be effective if frequently-occurring features are also highly discriminative, but this property does not hold across all datasets. Standard feature hashing achieves poor recovery error since colliding features cannot be disambiguated.

In Fig.~\ref{fig:l2err_reg}, we compare recovery error on \texttt{RCV1} across different settings of $\lambda$. Higher $\ell_2$-regularization results in less recovery error since both the true weights and the sketched weights are closer to 0; however, $\lambda$ settings that are too high can result in increased classification error.

\begin{figure*}
  \includegraphics[width=\textwidth]{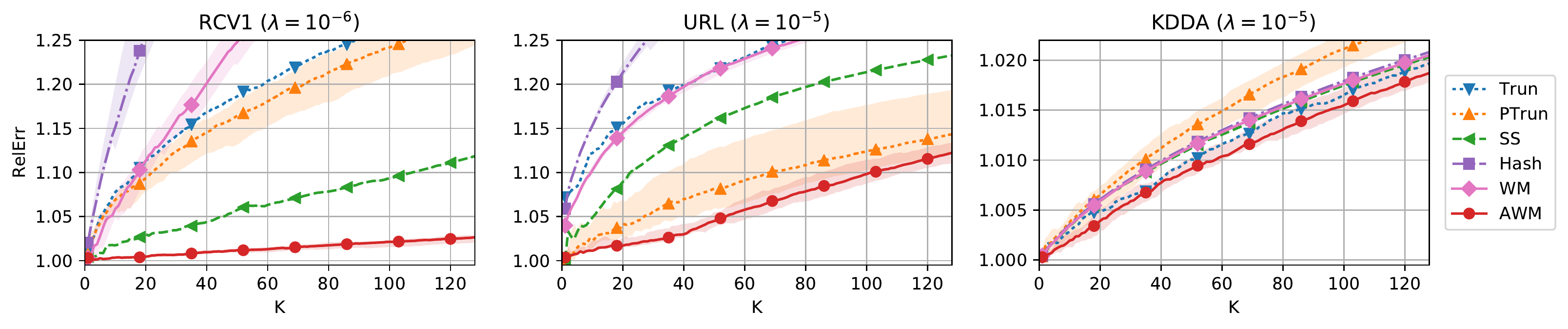}
  \caption{Relative $\ell_2$ error of estimated top-$K$ weights vs. true top-$K$ weights for $\ell_2$-regularized logistic regression under 8KB memory budget. Shaded area indicates range of errors observed over 10 trials. The AWM-Sketch achieves lower recovery error across all three datasets.}
  \label{fig:l2err_comp}
\end{figure*}

\begin{figure*}
  \includegraphics[width=\textwidth]{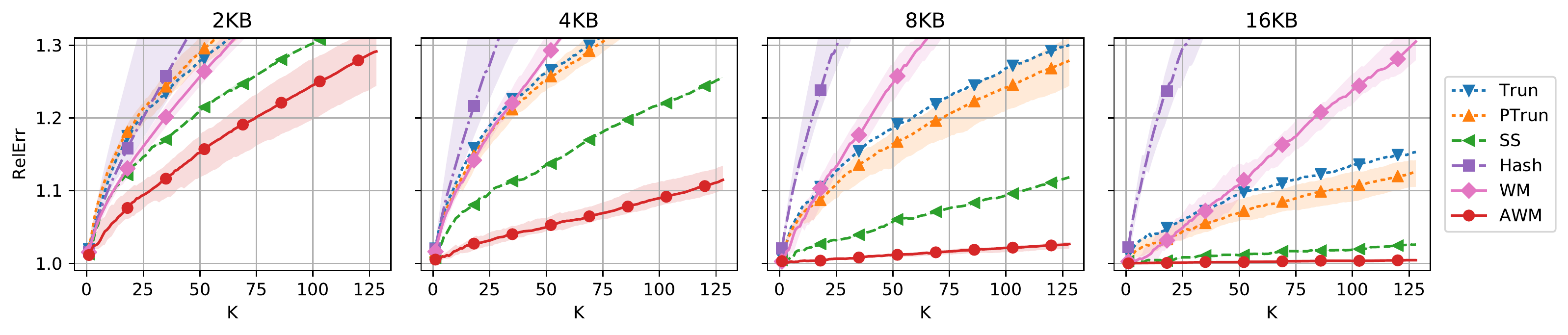}
  \caption{Relative $\ell_2$ error of estimated top-$K$ weights on \texttt{RCV1} dataset under different memory budgets ($\lambda = 10^{-6}$). Shaded area indicates range of errors observed over 10 trials. The recovery quality of the AWM-Sketch quickly improves with more allocated space.}
  \label{fig:l2err_space}
\end{figure*}

\begin{figure}
  \includegraphics[width=0.75\textwidth]{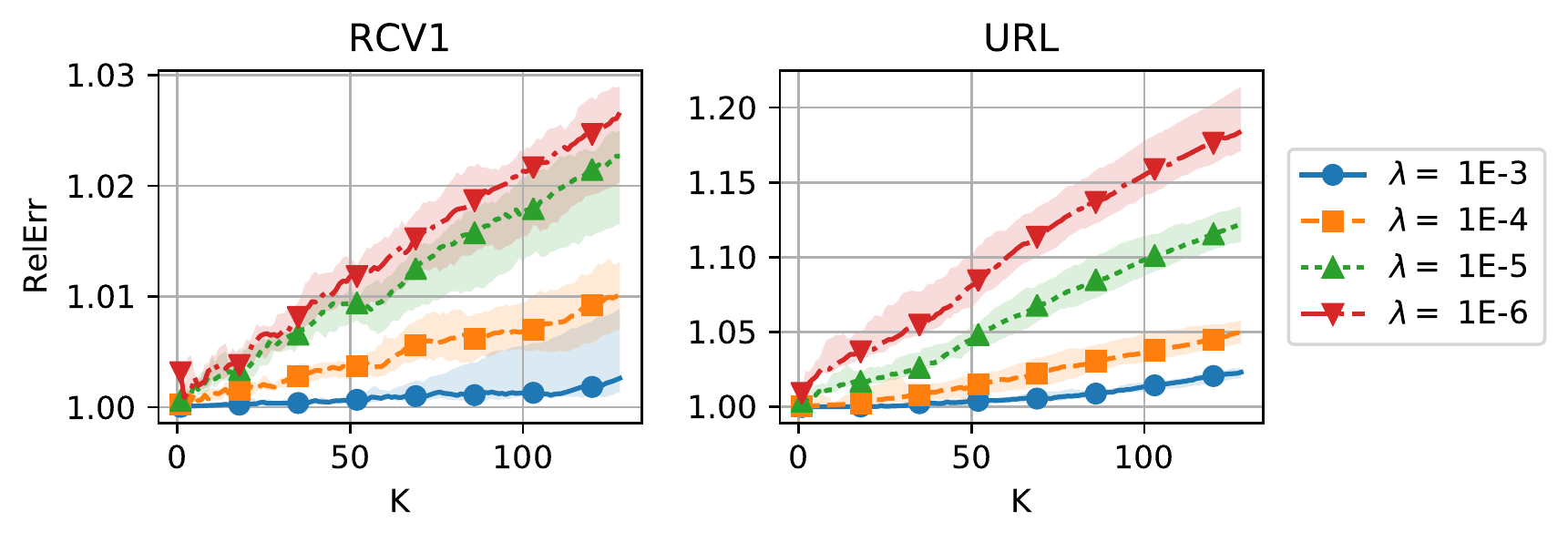}\centering
  \caption{Relative $\ell_2$ error of top-$K$ AWM-Sketch estimates with varying regularization parameter $\lambda$ on \texttt{RCV1} and \texttt{URL} datasets under 8KB memory budget.}
  \label{fig:l2err_reg}
\end{figure}

\subsection{Classification Error Rate}

\begin{figure*}
  \includegraphics[width=\textwidth]{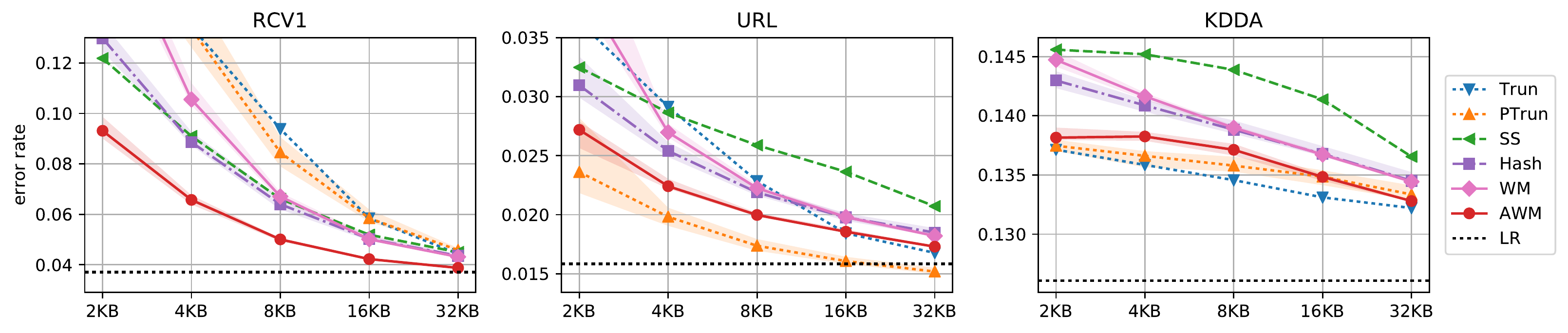}
  \caption{Online classification error rates with $\ell_2$-regularized logistic regression under different memory budgets (\texttt{Trun} = Simple Truncation, \texttt{PTrun} = Probabilistic Truncation, \texttt{SS} = Space Saving Frequent, \texttt{Hash} = Feature Hashing, \texttt{LR} = Logistic Regression without memory constraints). The AWM-Sketch consistently achieves better classification accuracy than methods that track frequent features.}
  \label{fig:err-comp}
\end{figure*}

We evaluated the classification performance of our models by measuring their online error rate \citep{blum1999beating}: for each observed pair $(\mathbf{x}_t, y_t)$, we record whether the prediction $\hat{y}_t$ (made without observing $y_t$) is correct before updating the model. The error rate is defined as the cumulative number of mistakes made divided by the number of iterations. Our results here are summarized in Fig.~\ref{fig:err-comp}. For each dataset, we used the value of $\lambda$ that achieved the lowest error rate across all our memory-limited methods. For each method and for each memory budget, we chose the configuration that achieved the lowest error rate. For the WM-Sketch, this corresponded to a width of $2^7$ or $2^8$ with depth scaling proportionally with the memory budget; for the AWM-Sketch, the configuration that uniformly performed best allocated half the space to the active set and the remainder to a depth-1 sketch (i.e., a single hash table without any replication).

We found that across all tested memory constraints, the AWM-Sketch consistently achieved lower error rate than heavy-hitter-based methods. Surprisingly, the AWM-Sketch outperformed feature hashing by a small but consistent margin: 0.5--3.7\% on \texttt{RCV1}, 0.1--0.4\% on \texttt{URL}, and 0.2--0.5\% on \texttt{KDDA}, with larger gains seen at smaller memory budgets. This suggests that the AWM-Sketch benefits from the precise representation of the largest, most-influential weights in the model, and that these gains are sufficient to offset the increased collision rate due to the smaller hash table. The Space Saving baseline exhibited inconsistent performance across the three datasets, demonstrating that tracking the most frequent features is an unreliable heuristic: features that occur frequently are not necessarily the most predictive. We note that higher values of the regularization parameter $\lambda$ correspond to greater penalization of rarely-occurring features; therefore, we would expect the Space Saving baseline to better approximate the performance of the unconstrained classifier as $\lambda$ increases.

\subsection{Runtime Performance}

\begin{figure}\centering
  \includegraphics[width=0.75\textwidth]{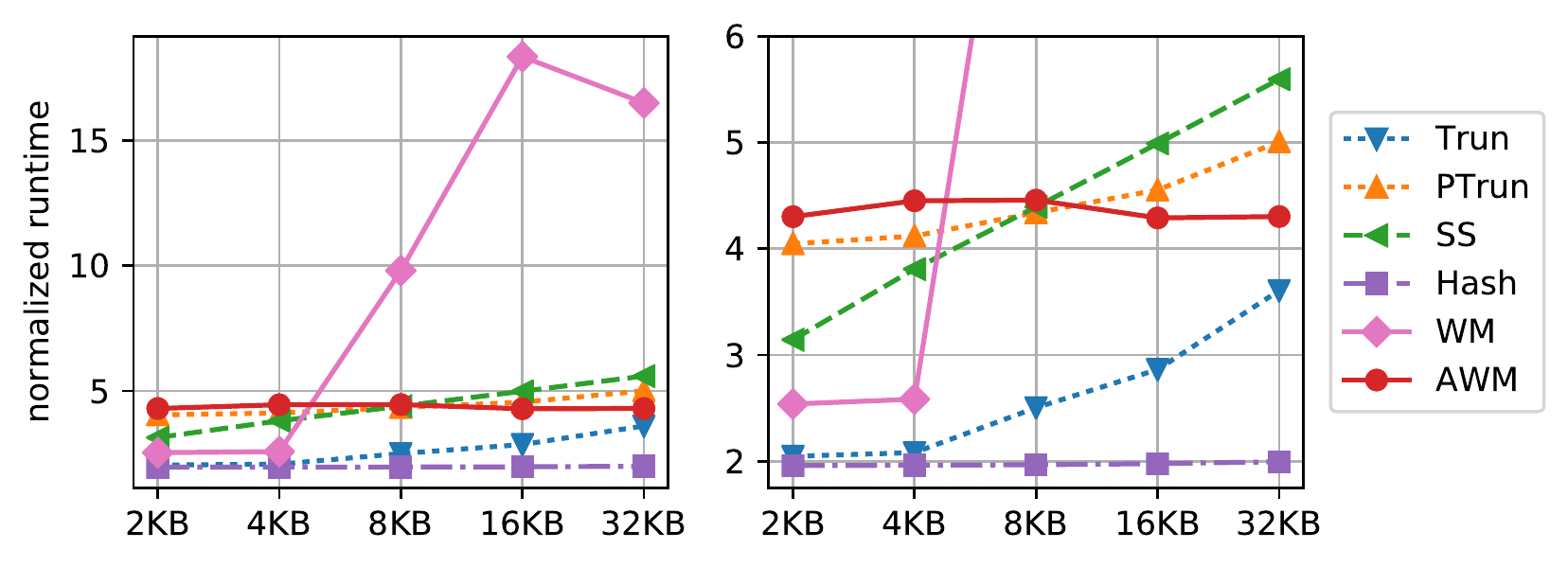}
  \caption{Normalized runtime of each method vs. memory-unconstrained logistic regression on \texttt{RCV1} using configurations that minimize recovery error (see Table~\ref{tab:l2_configs}). The right panel is a zoomed-in view of the left panel.}
  \label{fig:runtime}
\end{figure}

We evaluated runtime performance relative to a memory unconstrained logistic regression model using the same configurations as those chosen to minimize $\ell_2$ recovery error (Table~\ref{tab:l2_configs}). In all our timing experiments, we ran our implementations of the baseline methods, the WM-Sketch, and the AWM-Sketch on Intel Xeon E5-2690 v4 processor with 35MB cache using a single core. The memory-unconstrained logistic regression weights were stored using an array of 32-bit floating point values of size equal to the dimensionality of the feature space, with the highest-weighted features tracked using a min-heap of size $K=128$; reads and writes to the weight vector therefore required single array accesses. The remaining methods tracked heavy weights alongside 32-bit feature identifiers using a min-heap sized according to the corresponding configuration.

In our experiments, the fastest method was feature hashing, with about a 2x overhead over the baseline. This overhead was due to the additional hashing step needed for each read and write to a feature index. The AWM-Sketch incurred an additional 2x overhead over feature hashing due to more frequent heap maintenance operations.

\section{Applications}
\label{sec:applications}

We now show that a variety of tasks in stream processing can be framed as memory-constrained classification. The unifying theme between these applications is that classification is a useful abstraction whenever the use case calls for \emph{discriminating} between streams or between subpopulations of a stream. These distinct classes can be identified by partitioning a single stream into quantiles (Sec.~\ref{sec:explanation}), comparing separate streams (Sec.~\ref{sec:monitoring}), or even by generating \emph{synthetic} examples to be distinguished from real samples (Sec.~\ref{sec:pmi}).

\subsection{Streaming Explanation}
\label{sec:explanation}

In data analysis workflows, it is often necessary to identify characteristic attributes that are particularly indicative of a given subset of data \citep{meliou-tutorial}. For example, in order to diagnose the cause of anomalous readings in a sensor network, it is helpful to identify common features of the outlier points such as geographical location or time of day. This use case has motivated the development of methods for finding common properties of outliers found in aggregation queries \citep{wu2013scorpion} and in data streams \citep{bailis2017macrobase}.

This task can be framed as a classification problem: assign positive labels to the outliers and negative labels to the inliers, then train a classifier to discriminate between the two classes. The identification of characteristic attributes is then reduced to the problem of identifying heavily-weighted features in the trained model. In order to identify indicative \emph{conjunctions} of attributes, we can simply augment the feature space to include arbitrary combinations of singleton features.

\begin{figure}\centering
  \includegraphics[width=0.75\textwidth]{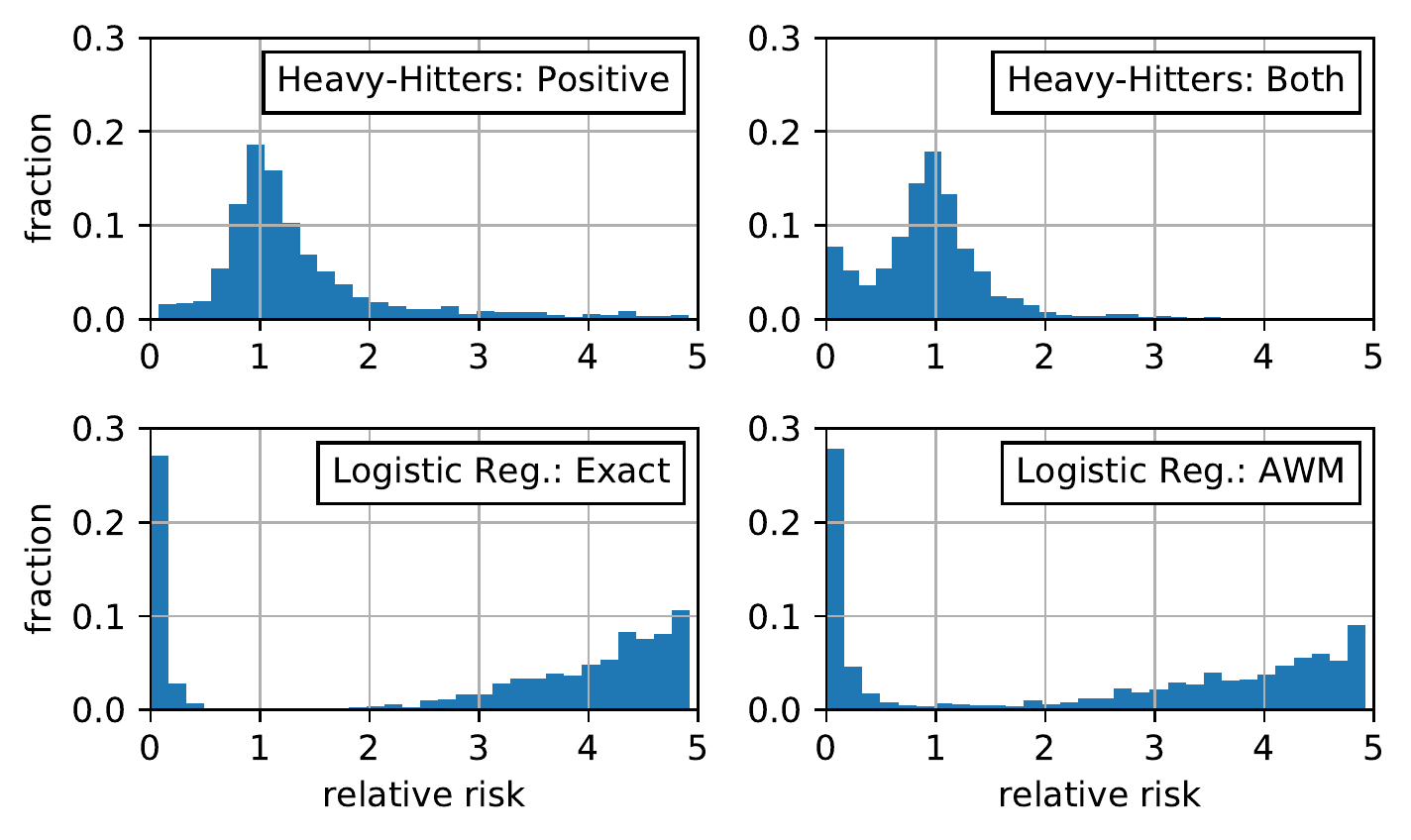}
  \caption{Distribution of relative risks among top-2048 features retrieved by each method. \emph{Top Row}: Heavy-Hitters. \emph{Bottom Row}: Classifier-based methods.}
  \label{fig:comp-explanation}
\end{figure}

\begin{figure}\centering
  \includegraphics[width=0.75\textwidth]{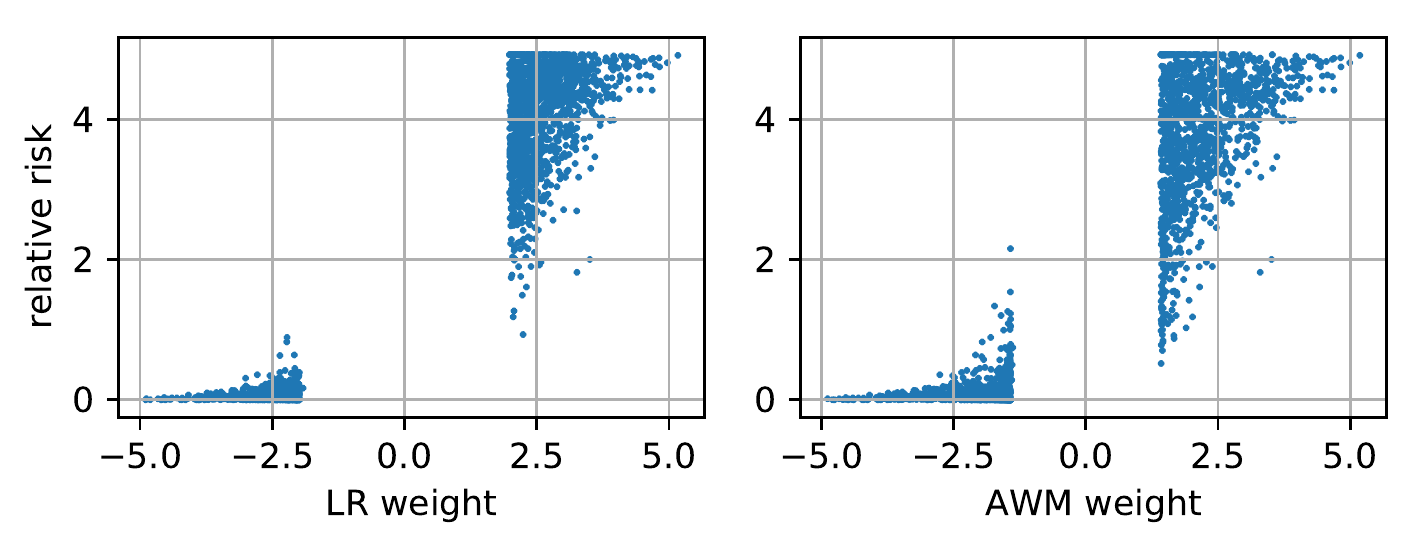}
  \caption{Correlation between top-2048 feature weights and relative risk. \emph{Left}: Memory-unconstrained logistic regression (Pearson correlation $0.95$). \emph{Right}: AWM-Sketch (Pearson correlation $0.91$).}
  \label{fig:campaign-scatter}
\end{figure}

The \emph{relative risk} or \emph{risk ratio} $r_x = p(y=1\mid x=1) / p(y=1 \mid x=0)$ is a statistical measure of the relative occurrence of the positive label $y=1$ when the feature $x$ is active versus when it is inactive. In the context of stream processing, the relative risk has been used to quantify the degree to which a particular attribute or attribute combination is indicative of a data point being an outlier relative to the overall population \citep{bailis2017macrobase}. Here, we are interested in comparing our classifier-based approach to identifying high-risk features against the approach used in \texttt{MacroBase} \citep{bailis2017macrobase}, an existing system for explaining outliers over streams, that identifies candidate attributes using a variant of the Space Saving heavy-hitters algorithm.

\minihead{Experimental Setup} We used a publicly-available dataset of itemized disbursements by candidates in U.S. House and Senate races from 2010--2016.\footnote{FEC candidate disbursements data: \url{http://classic.fec.gov/data/CandidateDisbursement.do}} The outlier points were set to be the set of disbursements in the top-20\% by dollar amount. For each row of the data, we generated a sequence of 1-sparse feature vectors\footnote{We can also generate a single feature vector per row (with sparsity greater than 1), but the learned weights would then correlate more weakly with the relative risk. This is due to the effect of correlations between features.} corresponding to the observed attributes. We set a space budget of 32KB for the AWM-Sketch.

\minihead{Results} Our results are summarized in Figs.~\ref{fig:comp-explanation} and \ref{fig:campaign-scatter}. The former empirically demonstrates that the heuristic of filtering features on the basis of frequency can be suboptimal for a fixed memory budget. This is due to features that are frequent in both the inlier and outlier classes: it is wasteful to maintain counts for these items since they have low relative risk. In Fig.~\ref{fig:comp-explanation}, the top row shows the distribution of relative risks among the most frequent items within the positive class (left) and across both classes (right). In contrast, our classifier-based approaches use the allocated space more efficiently by identifying features at the extremes of the relative risk scale. 

In Fig.~\ref{fig:campaign-scatter}, we show that the learned classifier weights are strongly correlated with the relative risk values estimated from true counts. Indeed, logistic regression weights can be interpreted in terms of log odds ratios, a related quantity to relative risk. These results show that the AWM-Sketch is a superior filter compared to Heavy-Hitters approaches for identifying high-risk features.

\subsection{Network Monitoring}
\label{sec:monitoring}

IP network monitoring is one of the primary application domains for sketches and other small-space summary methods \citep{venkataraman2005new,bandi2007fast,yu2013software}. Here, we focus on the problem of finding packet-level features (for instance, source/destination IP addresses and prefixes, port numbers, network protocols, and header or payload characteristics) that differ significantly in relative frequency between a pair of network links. 

This problem of identifying significant relative differences---also known as \emph{relative deltoids}---was studied by Cormode and Muthukrishnan \citep{cormode2005s}. Concretely, the problem is to estimate---for each item $i$---ratios $\phi(i) = n_1(i) / n_2(i)$ (where $n_1, n_2$ denote occurrence counts in each stream) and to identify those items $i$ for which this ratio, or its reciprocal, is large.  Here, we are interested in identifying differences between traffic streams that are observed \emph{concurrently}; in contrast, the empirical evaluation in \citep{cormode2005s} focused on comparisons between different time periods. 

\minihead{Experimental Setup} We used a subset of an anonymized, publicly-available passive traffic trace dataset recorded at a peering link for a large ISP \citep{caida}. The positive class was the stream of outbound source IP addresses and the negative class was the stream of inbound destination IP addresses. We compared against several baseline methods, including ratio estimation using a pair of Count-Min sketches (as in \citep{cormode2005s}). For each method we retrieved the top-2048 features (i.e., IP addresses in this case) and computed the recall against the set of features above the given ratio threshold, where the reference ratios were computed using exact counts.

\minihead{Results} We found that the AWM-Sketch performed comparably to the memory-unconstrained logistic regression baseline on this benchmark. We significantly outperformed the paired Count-Min baseline by a factor of over 4$\times$ in recall while using the same memory budget, as well as a paired CM baseline that was allocated 8x the memory budget. These results indicate that linear classifiers can be used effectively to identify relative deltoids over pairs of data streams.

\begin{figure}\centering 
  \includegraphics[width=0.75\textwidth]{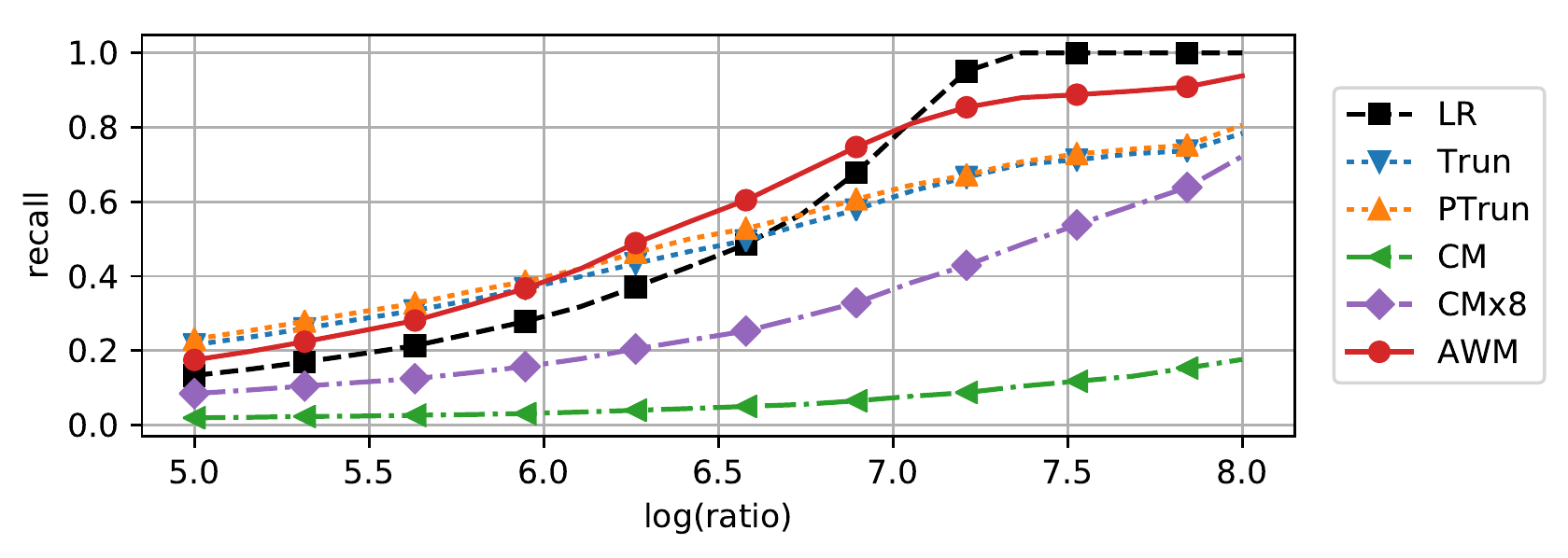}
  \caption{Recall of IP addresses with relative occurrence ratio above the given threshold with 32KB of space. \texttt{LR} denotes recall by full memory-unconstrained logistic regressor. \texttt{CMx8} denotes Count-Min baseline with 8x memory usage.}
  \label{fig:packet}
\end{figure}

\subsection{Streaming Pointwise Mutual Information}
\label{sec:pmi}

Pointwise mutual information (PMI), a measure of the statistical correlation between a pair of events, is defined as:
\begin{equation*}
  \mathrm{PMI}(x, y) = \log\frac{p(x, y)}{p(x)p(y)}.
\end{equation*}
Intuitively, positive values of the PMI indicate events that are positively correlated, negative values indicate events that are negatively correlated, and a PMI of 0 indicates uncorrelated events.

In natural language processing, PMI is a frequently-used measure of word association \citep{turney2010frequency}. Traditionally, the PMI is estimated using empirical counts of unigrams and bigrams obtained from a text corpus. The key problem with this approach is that the number of bigrams in standard natural language corpora can grow very large; for example, we found $\sim$47M unique co-occurring pairs of tokens in a small subset of a standard newswire corpus. This combinatorial growth in the feature dimension is further amplified when considering higher-order generalizations of PMI.

More generally, streaming PMI estimation can be used to detect pairs of events whose occurrences are strongly correlated. For example, we can consider a streaming log monitoring use case where correlated events are potentially indicative of cascading failures or trigger events resulting in exceptional behavior in the system. Therefore, we expect that the techniques developed here should be useful beyond standard NLP applications.

\minihead{Sparse Online PMI Estimation} Streaming PMI estimation using approximate counting has previously been studied \citep{durme2009streaming}; however, this approach has the drawback that memory usage still scales linearly with the number of observed bigrams. Here, we explore streaming PMI estimation from a different perspective: we pose a binary classification problem over the space of bigrams with the property that the model weights asymptotically converge to an estimate of the PMI.\footnote{This classification formulation is used in the popular \texttt{word2vec} skip-gram method for learning word embeddings \citep{mikolov2013efficient}; the connection to PMI approximation was first observed by \citet{levy2014neural}. To our knowledge, we are the first to apply this formulation in the context of sparse PMI estimation.}

The classification problem is set up as follows: in each iteration $t$, with probability $0.5$ sample a bigram $(u, v)$ from the bigram distribution $p(u, v)$ and set $y_t = +1$; with probability $0.5$ sample $(u, v)$ from the unigram product distribution $p(u)p(v)$ and set $y_t = -1$. The input $\mathbf{x}_t$ is the 1-sparse vector where the index corresponding to $(u, v)$ is set to $1$. We train a logistic regression model to discriminate between the \emph{true} and \emph{synthetic} samples. If $\lambda = 0$, the model asymptotically converges to the distribution $\hat{p}(y=1 \mid (u, v)) = f(w_{uv}) = p(u, v) / \left(p(u, v) + p(u)p(v)\right)$ for all pairs $(u, v)$, where $f$ is the logistic function. It follows that $w_{uv} = \log (p(u, v) / p(u)p(v))$, which is exactly the PMI of $(u, v)$. If $\lambda > 0$, we obtain an estimate that is biased, but with reduced variance in the estimates for rare bigrams.

\begin{table}\centering
  \begin{tabular}{@{}lrrclr@{}}\toprule
  Pair & PMI & Est. & \phantom{a} & Pair & PMI \\
  \cmidrule{1-3} \cmidrule{5-6}
  prime minister & 6.339 & 7.609 && , the & 0.044 \\
  los angeles & 7.197 & 7.047 && the , & -0.082 \\
  http / & 6.734 & 7.001 && the of & 0.611 \\
  human rights &6.079 & 6.721 && the . & 0.057 \\
  \bottomrule
  \end{tabular}
  \caption{\textit{Left:} Top recovered pairs with PMI computed from true counts and PMI estimated from model weights ($2^{16}$ bins, 1.4MB total memory). \textit{Right:} Most common pairs in corpus.}
  \label{tab:pmi-output}
\end{table}

\begin{figure}\centering
  \includegraphics[width=0.75\textwidth]{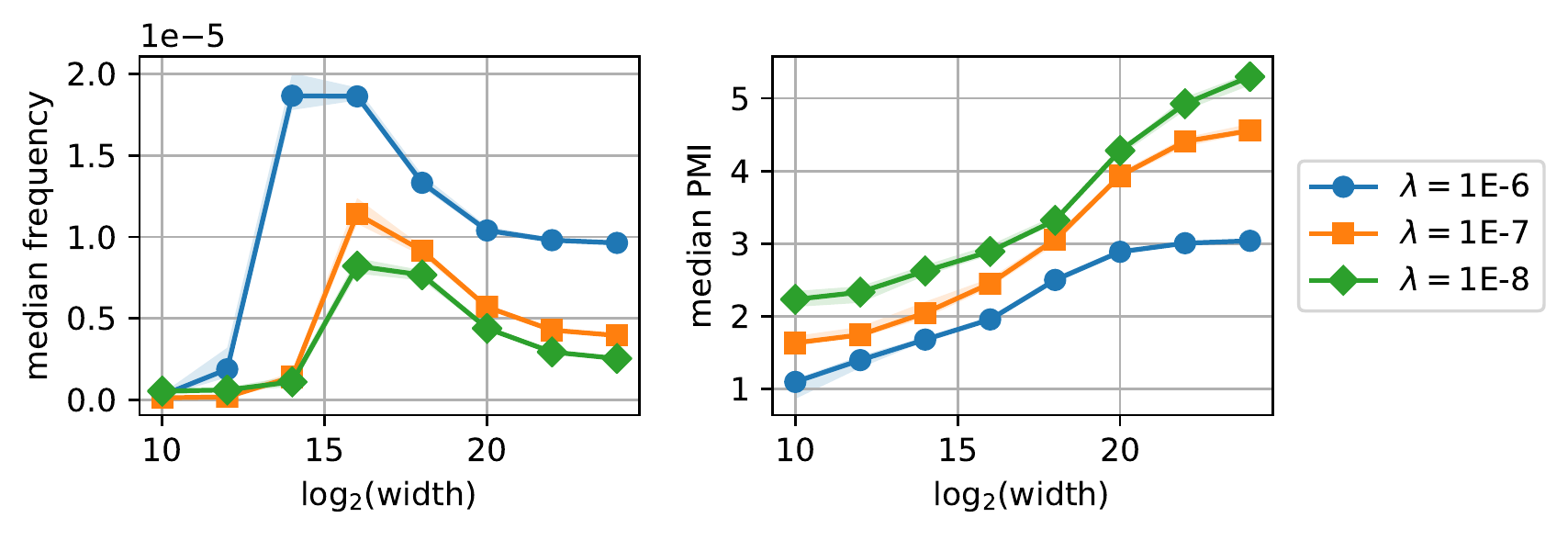}
  \caption{Median frequencies and exact PMIs of retrieved pairs with Active-Set Weight-Median Sketch estimation. Lower $\lambda$ settings and higher bin counts favor less frequent pairs.}
  \label{fig:pmi-plots}
\end{figure}

\minihead{Experimental Setup} We train on a subset of a standard newswire corpus \citep{chelba2013one}; the subset contains 77.7M tokens, 605K unique unigrams and 47M unique bigrams over a sliding window of size 6. In our implementation, we approximate sampling from the unigram distribution by sampling from a reservoir sample of tokens \citep{may2017streaming,kaji2017incremental}. We estimated weights using the AWM-Sketch with heap size 1024 and depth 1; the reservoir size was fixed at 4000. We make a single pass through the dataset and generate 5 negative samples for every true sample. Strings were first hashed to 32-bit values using \texttt{MurmurHash3};\footnote{\url{https://github.com/aappleby/smhasher/wiki/MurmurHash3}} these identifiers were hashed again to obtain sketch bucket indices.

\minihead{Results} For width settings up to $2^{16}$, our implementation's total memory usage was at most $1.4$MB. In this regime, memory usage was dominated by the storage of strings in the heap and the unigram reservoir. For comparison, the standard approach to PMI estimation requires 188MB of space to store exact 32-bit counts for all bigrams, excluding the space required for storing strings or the token indices corresponding to each count.  In Table~\ref{tab:pmi-output}, we show sample pairs retrieved by our method; the PMI values estimated from exact counts are well-estimated by the classifier weights. In Fig.~\ref{fig:pmi-plots}, we show that at small widths, the high collision rate results in the retrieval of noisy, low-PMI pairs; as the width increases, we retrieve higher-PMI pairs which typically occur with lower frequency. Further, regularization helps discard low-frequency pairs but can result in the model missing out on high-PMI but less-frequent pairs.

\section{Discussion}
\label{sec:discussion}

\vspace{-2pt}
\minihead{Active Set vs. Multiple Hashing} In the basic WM-Sketch, multiple hashing is needed in order to disambiguate features that collide in a heavy bucket; we should expect that features with truly high weight should correspond to large values in the majority of buckets that they hash to. The active set approach uses a different mechanism for disambiguation. Suppose that all the features that hash to a heavy bucket are added to the active set; we should expect that the weights for those features that were erroneously added will eventually decay (due to $\ell_2$-regularization) to the point that they are evicted from the active set. Simultaneously, the truly high-weight features are retained in the active set. The AWM-Sketch can therefore be interpreted as a variant of feature hashing where the highest-weighted features are not hashed.

\minihead{The Cost of Interpretability}
We initially motivated the development of the WM-Sketch with the dual goals of memory-efficient learning and model interpretability. A natural question to ask is: what is the cost of interpretability? What do we sacrifice in classification accuracy when we allocate memory to storing feature identifiers relative to feature hashing, which maintains only feature weights?
A surprising finding in our evaluation on standard binary classification datasets was that the AWM-Sketch consistently improved on the classification accuracy of feature hashing. We hypothesize that the observed gains are due to reduced collisions with heavily-weighted features.
Notably, we are able to improve model interpretability by identifying important features without sacrificing any classification accuracy.

\minihead{Per-Feature Learning Rates} In previous work on online learning applications, practitioners have found that the per-feature learning rates can significantly improve classification performance \citep{mcmahan2013ad}. An open question is whether variable learning rate across features is worth the associated memory cost in the streaming setting.

\minihead{Multiclass Classification} The WM-Sketch be extended to the multiclass setting using the following simple extension. Given $M$ output classes, maintain $M$ copies of the WM-Sketch. In order to predict the output, we evaluate the output on each copy and return the maximum. For large $M$, for instance in language modeling applications, this procedure can be computationally expensive since update time scales linearly with $M$. In this regime, we can apply noise contrastive estimation \citep{gutmann2010noise}---a standard reduction to binary classification---to learn the model parameters.

\minihead{Further Extensions} The WM-Sketch is likely amenable to further extensions that can improve update time and further reduce memory usage. Our method is equivalent to online gradient descent using random projections of input features, and gradient updates can be performed asynchronously with relaxed cache coherence requirements between cores \citep{recht2011hogwild}. Additionally, our methods are orthogonal to reduced-precision techniques like randomized rounding \citep{raghavan1987randomized} and approximate counting \citep{flajolet1985approximate}; these methods can be used in combination to realize further memory savings.

\vspace*{-4pt}
\section{Conclusions}
\label{sec:conclusions}

In this paper, we introduced the Weight-Median Sketch for the problem of identifying heavily-weighted features in linear classifiers over streaming data. We showed theoretical guarantees for our method, drawing on techniques from online learning and norm-preserving random projections. In our empirical evaluation, we showed that the active set extension to the basic WM-Sketch achieves superior weight recovery and competitive classification error compared to baseline methods across several standard binary classification benchmarks. Finally, we explored promising applications of our methods by framing existing stream processing tasks as classification problems. We believe this machine learning perspective on sketch-based stream processing may prove to be a fruitful direction for future research in advanced streaming analytics.

\section*{Acknowledgments}
We thank Daniel Kang, Sahaana Suri, Pratiksha Thaker, and the anonymous reviewers for their feedback on earlier drafts of this work. This research was supported in part by affiliate members and other supporters of the Stanford DAWN
project---Google, Intel, Microsoft, Teradata, and VMware---as well as
industrial gifts and support from Toyota Research Institute, Juniper Networks, Keysight Technologies,
Hitachi, Facebook, Northrop Grumman, and NetApp.
The authors were also supported by 
DARPA under No. FA8750-17-2-0095
(D3M), ONR Award N00014-17-1-2562, NSF Award CCF-1704417, and a Sloan Research Fellowship.

\bibliography{references} 

\appendix
\section{Proofs}
\subsection{Proof of Theorem \ref{thm:batch-recovery}}\label{sec:proofs}

We will use the duals of $L(\b{w})$ and $\hat{L}(\b{z})$ to show that $\b{z}_*$ is close to $R\b{w}_*$. Some other works such as \citet{zhang2014random} and \citet{yang2015theory} have also attempted to analyze random projections via the dual, and this has the advantage that the dual variables are often easier to compare, as they at least have the same dimensionality. Note that $R\b{w}_*$ is essentially the Count-Sketch projection of $\b{w}_*$, hence showing that $\b{z}_*$ is close to $R\b{w}_*$ will allow us to show that doing count-sketch recovery using $\b{z}_*$ is comparable to doing Count-Sketch recovery from the projection of $\b{w}_*$ itself, and hence give us the desired error bounds. We first derive the dual forms of the objective function $L(\b{w})$, the dual of $\hat{L}(\b{z})$ can be derived analogously. Let $u_i = y_i\b{w}^T\b{x}_i$. Then we can write the primal as:
\begin{align*}
\min_{\b{u},\b{w}}\quad & \frac{1}{T}\sum_{i=1}^T \ell(u_i) + \frac{\lambda}{2}\normsq{\b{w}},\\
\text{subject to}\quad & u_i = y_i\b{w}^T\b{x}_i, \; \forall \; i \le T.
\end{align*}
Define $\tilde{X}_i = y_i \b{x}_i$, i.e. the $i$th data point $\b{x}_i$ times its label. Let $\tilde{X}\in \mathbb{R}^{d \times T}$ be the matrix of data points such that the $i$th column is $\tilde{X}_i$. Let $K = \tilde{X}^T \tilde{X}$ be the kernel matrix corresponding to the original data points. Taking the Lagrangian, and minimizing with respect to the primal variables $\b{z}$ and $\b{w}$ gives us the following dual objective function in terms of the dual variable $\b{\alpha}$:
\begin{align*}
J(\b{\alpha}) &= \frac{1}{T}\sum_i \ell^*(\b{\alpha}_i) + \frac{1}{2\lambda T^2}\b{\alpha}^T K \b{\alpha},
\end{align*}
where $\ell^*$ is the Fenchel conjugate of $\ell$. Also, if $\b{\alpha}_*$ is the minimizer of $J(\b{\alpha})$, then the minimizer $\b{w}_*$ of $L(\b{w})$ is given by $\b{w}_* = -\frac{1}{\lambda T} \tilde{X} \b{\alpha}_*$. Similarly, let $K = \tilde{X}^T R^T R \tilde{X}$ be the kernel matrix corresponding to the projected data points. We can write down the dual $\hat{L}(\b{\alpha})$ of the projected primal objective function $\hat{J}(\b{w})$  in terms of the dual variable $\hat{\b{\alpha}}$ as follows:
\begin{align*}
\hat{J}(\hat{\b{\alpha}}) &= \frac{1}{T}\sum_i \ell^*(\hat{\b{\alpha}}_i) + \frac{1}{2\lambda T^2}\hat{\b{\alpha}}^T \hat{K} \hat{\b{\alpha}}.
\end{align*}
As before, if $\hat{\b{\alpha}}_*$ is the minimizer of $ \hat{J}(\hat{\b{\alpha}})$, then the minimizer $\b{z}_*$ of $\hat{L}(\b{z})$ is given by $\hat{\b{w}}_* = -\frac{1}{\lambda T} R \tilde{X} \hat{\b{\alpha}}_*$. We will first express the distance between $\b{z}_*$ and $R\b{w}_*$ in terms of the distance between the dual variables. We can write:
\begin{align}
\normsq{\b{z}_*- R\b{w}_*} &= \frac{1}{\lambda^2T^2} \normsq{R\tilde{X}\hat{\b{\alpha}}_* - R\tilde{X}\b{\alpha}_*}\nonumber\\
&= \frac{1}{\lambda^2T^2} (\hat{\b{\alpha}}_*-\b{\alpha}_*)^T \hat{K} (\hat{\b{\alpha}}_*-\b{\alpha}_*)\label{eq:relate1}.
\end{align}
Hence our goal will be to upper bound $(\hat{\b{\alpha}}_*-\b{\alpha}_*)^T \hat{K} (\hat{\b{\alpha}}_*-\b{\alpha}_*)$. Define $\b{\Delta} = \frac{1}{\lambda T}(\hat{K}-K) \b{\alpha}_*$. We will show that 
$(\hat{\b{\alpha}}_*-\b{\alpha}_*)^T \hat{K} (\hat{\b{\alpha}}_*-\b{\alpha}_*)$ can be upper bounded in terms of $\b{\Delta}$ as follows.
\begin{lemma}\label{lemma:delta}
	\begin{align*}
	\frac{1}{\lambda T^2}(\hat{\b{\alpha}}_*-\b{\alpha}_*)^T \hat{K} (\hat{\b{\alpha}}_*-\b{\alpha}_*) &\le \frac{1}{\lambda T^2}(\b{\alpha}_*-\hat{\b{\alpha}}_*)^T(\hat{K}-K) \b{\alpha}_*\\
	&\le \frac{1}{T} \|\hat{\b{\alpha}}_*-\b{\alpha}_*\|_1 \| \b{\Delta} \|_{\infty}.
	\end{align*}
\end{lemma}
\begin{proof}
We first claim that,
\begin{align}
\hat{J}(\b{\alpha}_*) \ge \hat{J}(\hat{\b{\alpha}}_*) + \frac{1}{2\lambda T^2}(\hat{\b{\alpha}}_*-\b{\alpha}_*)^T \hat{K} (\hat{\b{\alpha}}_*-\b{\alpha}_*).\label{eq:1}
\end{align}
Note that $\hat{J}(\b{\alpha}_*) \ge \hat{J}(\hat{\b{\alpha}}_*)$ just because $\hat{\b{\alpha}}_*$ is the minimizer of $\hat{J}(\hat{\b{\alpha}})$, hence Eq. \ref{eq:1} is essentially giving an improvement over this simple bound. In order to prove Eq. \ref{eq:1}, define $F(\b{\alpha}) = \frac{1}{T}\sum_i \ell^*(\b{\alpha}_i)$. As $F(\b{\alpha})$ is a convex function (because $\ell^*(x)$ is convex in $x$), from the definition of convexity,
\begin{align}
    F(\b{\alpha}_*) \ge F(\hat{\b{\alpha}}_*) + \langle \nabla F(\hat{\b{\alpha}}_*), \b{\alpha}_*-\hat{\b{\alpha}}_* \rangle. \label{eq:2}
\end{align}
It is easy to verify that,
\begin{align}
    \frac{1}{2\lambda T^2}\b{\alpha}_*^T \hat{K} \b{\alpha}_* = \frac{1}{2\lambda T^2}\hat{\b{\alpha}}_*^T \hat{K} \hat{\b{\alpha}}_* + \frac{1}{\lambda T^2}(\b{\alpha}_*-\hat{\b{\alpha}}_*)^T \hat{K} \hat{\b{\alpha}}_* \nonumber \\
    + \frac{1}{2\lambda T^2}(\hat{\b{\alpha}}_*-\b{\alpha}_*)^T \hat{K} (\hat{\b{\alpha}}_*-\b{\alpha}_*).\label{eq:3}
\end{align}
Adding Eq. \ref{eq:2} and Eq. \ref{eq:3}, we get,
\begin{align}
    \hat{J}({\b{\alpha}_*}) &\ge \hat{J}(\hat{\b{\alpha}}_*) + \langle \nabla F(\hat{\b{\alpha}}_*), \b{\alpha}_*-\hat{\b{\alpha}}_* \rangle + \frac{1}{\lambda T^2}(\b{\alpha}_*-\hat{\b{\alpha}}_*)^T \hat{K} \hat{\b{\alpha}}_* \nonumber\\
    &\quad + \frac{1}{2\lambda T^2}(\hat{\b{\alpha}}_*-\b{\alpha}_*)^T \hat{K} (\hat{\b{\alpha}}_*-\b{\alpha}_*)\nonumber \\
    &=  \hat{J}(\hat{\b{\alpha}}_*) + \langle \nabla \hat{J}(\hat{\b{\alpha}}_*), \b{\alpha}_*-\hat{\b{\alpha}}_* \rangle + \frac{1}{2\lambda T^2}(\hat{\b{\alpha}}_*-\b{\alpha}_*)^T \hat{K} (\hat{\b{\alpha}}_*-\b{\alpha}_*) \nonumber\\
    &= \hat{J}(\hat{\b{\alpha}}_*) + \frac{1}{2\lambda T^2}(\hat{\b{\alpha}}_*-\b{\alpha}_*)^T \hat{K} (\hat{\b{\alpha}}_*-\b{\alpha}_*),\nonumber
\end{align}
which verifies Eq. \ref{eq:1}. We can derive a similar bound for $\hat{J}(\hat{\b{\alpha}}_*)$,
\begin{align}
    \hat{J}(\hat{\b{\alpha}}_*) \ge \hat{J}(\b{\alpha}_*) + (\hat{\b{\alpha}}_*-\b{\alpha}_*) ^T \nabla \hat{J}(\b{\alpha}_*) + \frac{1}{2\lambda T^2}(\hat{\b{\alpha}}_*-\b{\alpha}_*)^T \hat{K} (\hat{\b{\alpha}}_*-\b{\alpha}_*). \label{eq:bound_diff}
\end{align}
As $J(\b{\alpha})$ is minimized by $\b{\alpha}_*$ and $J(\b{\alpha})$ is convex,
\begin{align}
    &(\hat{\b{\alpha}}_*-\b{\alpha}_*)^T \nabla J(\b{\alpha}_*) \ge 0 \nonumber\\
    \implies& (\hat{\b{\alpha}}_*-\b{\alpha}_*)^T (\nabla F(\b{\alpha}_*) + \frac{1}{\lambda T^2}K \b{\alpha}_*) \ge 0 \nonumber \\
    \implies& (\hat{\b{\alpha}}_*-\b{\alpha}_*)^T (\nabla \hat{J}(\b{\alpha}_*) + \frac{1}{\lambda T^2}({K}-\hat{K}) \b{\alpha}_*) \ge 0 \nonumber\\
    \implies& (\hat{\b{\alpha}}_*-\b{\alpha}_*)^T \nabla \hat{J}(\b{\alpha}_*) \ge \frac{1}{\lambda T^2}(\hat{\b{\alpha}}_*-\b{\alpha}_*)^T(\hat{K}-K) \b{\alpha}_*. \label{eq:bound_grad}
\end{align}
Using Eq. \ref{eq:bound_grad} in Eq. \ref{eq:bound_diff}, we get,
\begin{align}
    \hat{J}(\hat{\b{\alpha}}_*) - \hat{J}(\b{\alpha}_*) \ge \frac{1}{\lambda T^2}(\hat{\b{\alpha}}_*-\b{\alpha}_*)^T(\hat{K}-K) \b{\alpha}_*\nonumber\\
    + \frac{1}{2\lambda T^2}(\hat{\b{\alpha}}_*-\b{\alpha}_*)^T \hat{K} (\hat{\b{\alpha}}_*-\b{\alpha}_*). \nonumber
\end{align}
By Eq. \ref{eq:1}, $\hat{J}(\hat{\b{\alpha}}_*)-\hat{J}({\b{\alpha}_*})\le -\frac{1}{2\lambda T^2}(\hat{\b{\alpha}}_*-\b{\alpha}_*)^T \hat{K} (\hat{\b{\alpha}}_*-\b{\alpha}_*) $. Therefore we can write,
\begin{align*}
-\frac{1}{\lambda T^2}(\hat{\b{\alpha}}_*-\b{\alpha}_*)^T \hat{K} (\hat{\b{\alpha}}_*-\b{\alpha}_*) \ge \frac{1}{\lambda T^2}(\hat{\b{\alpha}}_*-\b{\alpha}_*)^T(\hat{K}-K) \b{\alpha}_*\\
\implies    \frac{1}{\lambda T^2}(\hat{\b{\alpha}}_*-\b{\alpha}_*)^T \hat{K} (\hat{\b{\alpha}}_*-\b{\alpha}_*) \le \frac{1}{\lambda T^2}(\b{\alpha}_*-\hat{\b{\alpha}}_*)^T(\hat{K}-K) \b{\alpha}_*.
\end{align*}
To finish, by Holder's inequality,
\[
\frac{1}{\lambda T^2}(\b{\alpha}_*-\hat{\b{\alpha}}_*)^T(\hat{K}-K) \b{\alpha}_*\le\frac{1}{\lambda T^2}\lpnorm{\b{\alpha}_*-\hat{\b{\alpha}}_*)}{1}\lpnorm{(\hat{K}-K)}{\infty} \b{\alpha}_*= \frac{1}{T} \|\hat{\b{\alpha}}_*-\b{\alpha}_*\|_1 \| \b{\Delta} \|_{\infty}
\]
\end{proof}
Our next step is to upper bound $\| \hat{\b{\alpha}}_*-{\b{\alpha}}_* \|_1$ in terms of $\| \b{\Delta} \|_{\infty}$ using the $\beta$-strongly smooth property of $\ell(z)$.
\begin{lemma}\label{lemma:norm1}
	\[
	\| \hat{\b{\alpha}}_*-{\b{\alpha}}_* \|_1 \le 2T\beta \| \b{\Delta} \|_{\infty}
	\]
\end{lemma}
\begin{proof}
We first claim that $J(w)$ is $1/(T\beta)$-strongly convex. Note that as $\ell(u_i)$ is $\beta$-strongly smooth, therefore, $\ell^*(\b{\alpha}_i)$ is $1/\beta$-strongly convex. Therefore for any $\b{u}$ and $\b{v}$,
\begin{align*}
\ell^*(u_i) &\ge \ell^*(v_i) + (u_i - v_i) \nabla\ell^*(v_i) + \frac{(u_i-v_i)^2}{2\beta}\\
\implies \frac{1}{T}\sum_i\ell^*(u_i) &\ge \frac{1}{T}\sum_i\ell^*(v_i) + \frac{1}{T} \langle \nabla\ell^*(\b{v}), \b{u} - \b{v} \rangle \\
&+ \frac{\|\b{u}-\b{v}\|_2^2}{2T\beta}.
\end{align*}
Therefore $F(\b{\alpha}) = \frac{1}{T}\sum_i\ell^*(\b{\alpha}_i)$ is $1/(T\beta)$ strongly convex. Note that $\frac{1}{2\lambda T^2}\b{\alpha}^T K \b{\alpha}$ is a convex function of $\b{\alpha}$ as $K=\tilde{X}^T \tilde{X}$ is a positive semidefinite matrix. Therefore $J(\b{\alpha}) = F(\b{\alpha}) + \frac{1}{2\lambda T^2}\b{\alpha}^T K \b{\alpha}$ is $1/(T\beta)$ strongly convex. It follows from the same reasoning that $\hat{J}(\hat{\b{\alpha}})$ is also $1/(T\beta)$-strongly convex.\\

By the definition of strong convexity,
\begin{align}
    \hat{J}(\hat{\b{\alpha}}_*) \ge \hat{J}({\b{\alpha}}_*) + (\hat{\b{\alpha}}_*-{\b{\alpha}}_*)^T \nabla \hat{J}(\b{\alpha}_*) + \frac{1}{2T\beta} \normsq{\hat{\b{\alpha}}_*-{\b{\alpha}}_*}.\nonumber
\end{align}
As $\hat{\b{\alpha}}_*$ is the minimizer of $\hat{J}({\b{\alpha}})$, $\hat{J}(\hat{\b{\alpha}}_*)\le \hat{J}({\b{\alpha}}_*)$. Therefore,
\begin{align*}
     &(\hat{\b{\alpha}}_*-{\b{\alpha}}_*)^T \nabla \hat{J}(\b{\alpha}_*) + \frac{1}{2T\beta} \normsq{\hat{\b{\alpha}}_*-{\b{\alpha}}_*} \le  \hat{J}(\hat{\b{\alpha}}_*)-\hat{J}({\b{\alpha}}_*)\le 0 \\
    \implies &\frac{1}{2T\beta} \normsq{\hat{\b{\alpha}}_*-{\b{\alpha}}_*} \le -(\hat{\b{\alpha}}_*-{\b{\alpha}}_*)^T \nabla \hat{J}(\b{\alpha}_*).
\end{align*}
Using Eq. \ref{eq:bound_grad}, we can rewrite this as,
\begin{align*}
    & \frac{1}{2T\beta} \normsq{\hat{\b{\alpha}}_*-{\b{\alpha}}_*}\le -\frac{1}{\lambda T^2}(\hat{\b{\alpha}}_*-\b{\alpha}_*)^T(\hat{K}-K)\b{\alpha}_*   \\
    &\le\frac{1}{T} \| \hat{\b{\alpha}}_*-{\b{\alpha}}_* \|_1 \| \b{\Delta} \|_{\infty}\\
    \implies &\normsq{\hat{\b{\alpha}}_*-{\b{\alpha}}_*}\le 2\beta\| \hat{\b{\alpha}}_*-{\b{\alpha}}_* \|_1 \| \b{\Delta} \|_{\infty}.
\end{align*}
By Cauchy-Schwartz, 
\begin{align*}
    \| \hat{\b{\alpha}}_*-{\b{\alpha}}_* \|_1^2 &\le T \| \hat{\b{\alpha}}_*-{\b{\alpha}}_* \|_2^2 \\
    &\le 2T\beta \| \hat{\b{\alpha}}_*-{\b{\alpha}}_* \|_1 \| \b{\Delta} \|_{\infty}\\
\implies \| \hat{\b{\alpha}}_*-{\b{\alpha}}_* \|_1 &\le 2T\beta \| \b{\Delta} \|_{\infty}.
\end{align*}

\end{proof}
We now bound $\| \b{\Delta}\|_{\infty}$. The result relies on the JL property of the projection matrix $R$ (recall Definition \ref{def:JL}). If $R$ is a JL matrix with error $\epsilon$ and failure probability $\delta/d^2$, then with failure probability $\delta$, for all coordinate basis vectors $\{\b{e}_1, \dots, \b{e}_d\}$,
\begin{align}
\norm{R\b{e}_i} = 1 \pm \epsilon,\; \forall \; i, \quad  |\langle R\b{e}_i, R\b{e}_j \rangle | \le \epsilon, \; \forall \; i \ne j\label{eq:ip_preserve}
\end{align}
The first bound follows directly from the JL property, to derive the bound on the inner product, we rewrite the inner product as follows,
\begin{align*}
4\langle R\b{e}_i, R\b{e}_j \rangle = \normsq{R\b{e}_i+R\b{e}_j}-\normsq{R\b{e}_i-R\b{e}_j}.
\end{align*}
The bound then follows by using the JL property for $\normsq{R\b{e}_i+R\b{e}_j}$ and $\normsq{R\b{e}_i-R\b{e}_j}$. As there are $d^2$ pairs, if the inner product is preserved for each pair with failure probability $\delta/d^2$, then by a union bound it is preserved simultaneously for all of them with failure probability $\delta$. Condition \ref{eq:ip_preserve} is useful because it implies that $R$ approximately preserves the inner products and lengths of sparse vectors, as states in the following Lemma:
\begin{lemma}\label{lem:jl_l1}
	If condition \ref{eq:ip_preserve} is satisfied, then for any two vectors $\b{v}_1$ and $\b{v}_2$, 
	\begin{align*}
	|\b{v}_1^T\b{v}_2 - (R\b{v}_1)^T(R\b{v}_2)| \le 2\epsilon \lpnorm{\b{v}_1}{1}\lpnorm{\b{v}_2}{1}.
	\end{align*}
\end{lemma}
\begin{proof}
	To verify, write $\b{v}_1 = \sum_{i=1}^{d}\b{\psi}_i \b{e}_i$ and $\b{v}_2 = \sum_{i=1}^{d}\b{\phi}_i \b{e}_i$. Then,
	\begin{align*}
	(R\b{v}_1)^T(R\b{v}_2) &= \Big(\sum_{i=1}^{d}\b{\psi}_i (R\b{e}_i)^T\Big)\Big(\sum_{j=1}^{d}\b{\phi}_j (R\b{e}_j)\Big) \\
	&= \sum_{i=1}^{d}\b{\psi}_i \b{\phi}_i \normsq{R\b{e}_i} + \sum_{i\ne j}^{}\b{\psi}_i \b{\phi}_j (R\b{e}_i)^T(R\b{e}_j)\\
	&\le (1+\epsilon) \sum_{i=1}^{d}\b{\psi}_i \b{\phi}_i + \sum_{i\ne j}^{}|\b{\psi}_i| |\b{\phi}_j| |(R\b{e}_i)^T(R\b{e}_j)|\\
	&\le (1+\epsilon)\b{v}_1^T\b{v}_2 + \epsilon\sum_{i\ne j}^{}|\b{\psi}_i| |\b{\phi}_j| \quad \text{(using condition \ref{eq:ip_preserve})} \\
	&\le \b{v}_1^T\b{v}_2 +\epsilon\norm{\b{v}_1}\norm{\b{v}_2} + \epsilon\sum_i|\b{\psi}_i| \sum_j|\b{\phi}_j|\\
	&\le \b{v}_1^T\b{v}_2 + 2\epsilon\lpnorm{\b{v}_1}{1}\lpnorm{\b{v}_2}{1}, 
	\end{align*}
	where the last line is due to $\sum_{i}^{}|\b{\psi}_i| = \lpnorm{\b{v}_1}{1}$ and $\norm{\b{v}_1}\le \lpnorm{\b{v}_1}{1}$. By a similar argument, we can show that $(R\b{v}_1)^T(R\b{v}_2)-\b{v}_1^T\b{v}_2 \ge 2\epsilon\lpnorm{\b{v}_1}{1}\lpnorm{\b{v}_2}{1}$. Hence for any two vectors $\b{v}_1$ and $\b{v}_2$, 
	\begin{align*}
	|\b{v}_1^T\b{v}_2 - (R\b{v}_1)^T(R\b{v}_2)| \le 2\epsilon \lpnorm{\b{v}_1}{1}\lpnorm{\b{v}_2}{1}.
	\end{align*}
\end{proof}

We can now bound $\|\b{\Delta}\|_{\infty}$ using Lemma \ref{lem:jl_l1}.

\begin{lemma}\label{lemma:bound_delta}
     If $R$ satisfies condition \ref{eq:ip_preserve} then,
    \[
    \|\b{\Delta}\|_{\infty}\le 2\gamma\epsilon \| \b{w}_* \|_1,
    \]
    where $\gamma= \max_i \|\b{x}_i\|_1$.
\end{lemma}
\begin{proof}
	We first rewrite $\b{\Delta}$ as follows,
	\begin{align*}
	\b{\Delta}& = \frac{1}{\lambda T} (\tilde{X}^TR^TR\tilde{X}-\tilde{X}^T\tilde{X})\b{\alpha}_*=\frac{1}{\lambda T} \tilde{X}^T(R^TR-I)\tilde{X}\b{\alpha}_*\\
	&= \tilde{X}^T(I-R^TR)\b{w}_*,
	\end{align*}
	using the relation that $\b{w}_* = -\frac{1}{\lambda T}\tilde{X}\b{\alpha}_*$. Therefore,
	\[
	\lpnorm{\b{\Delta}}{\infty} \le \max_i | \b{x}_i^T(I-R^TR)\b{w}_* | = \max_i | \b{x}_i^T\b{w}_*-(R\b{x}_i)^T(R\b{w}_*) |
	\]
Using Lemma \ref{lem:jl_l1}, it follows that,
\[
\lpnorm{\b{\Delta}}{\infty} \le \max_i | \b{x}_i^T\b{w}_*-(Rx_i)^T(R\b{w}_*) | \le 2\epsilon \gamma \lonenorm{\b{w}_*}.
\]
\end{proof}	

We will now combine Lemma \ref{lemma:delta}, \ref{lemma:norm1} and \ref{lemma:bound_delta}. By Eq. \ref{eq:relate1} and Lemma \ref{lemma:delta},
\[
\normsq{\b{z}_*-R\b{w}_*} \le  \frac{1}{\lambda T} \|\hat{\b{\alpha}}_*-\b{\alpha}_*\|_1 \| \b{\Delta} \|_{\infty}.
\]
Combining this with Lemma \ref{lemma:norm1},
\begin{align*}
\normsq{\b{z}_*-R\b{w}_*} \le \frac{2\beta}{\lambda} \| \b{\Delta} \|_{\infty}^2.
\end{align*}
If $R$ is a JL matrix with error $\epsilon$ and failure probability $\delta/d^2$, then by Lemma \ref{lemma:bound_delta}, with failure probability $\delta$,
\begin{align}
\norm{\b{z}_*-R\b{w}_*} \le \gamma\epsilon\sqrt{\frac{8\beta}{\lambda}} \| \b{w}_* \|_1.  \label{eq:e1}
\end{align}
By \citet{kane2014sparser}, the random projection matrix $R$ satisfies the JL property with error $\theta$ and failure probability $\delta'/d^2$ for $k\ge{C\log(d/\delta')/\theta^2}$, where $C$ is a fixed constant. Using Eq. \ref{eq:e1}, with failure probability $\delta'$,
\begin{align}
\norm{\b{z}_*-R\b{w}_*} \le 8\gamma\theta\sqrt{\frac{8\beta}{\lambda}} \| \b{w}_* \|_1.  \label{eq:e2}
\end{align}
Recall that $R\sqrt{s}$ is a count-sketch matrix with width $C_1/\theta$ and depth $s=C_2\log(d/\delta')/\theta$, where $C_1$ and $C_2$ are fixed constants. Let $\b{w}_{\mathrm{proj}}$ be the projection of $\b{w}_*$ with the count-sketch matrix $\tilde{R}$, hence $\b{w}_{\mathrm{proj}}=\sqrt{s}R\b{w}_*$. Let $\b{z}_{\mathrm{proj}}=\sqrt{s}\b{z}_*$. By Eq. \ref{eq:e2}, with failure probability $\delta'$,
\[
\norm{\b{z}_{\mathrm{proj}}-\b{w}_{\mathrm{proj}}} \le \sqrt{\frac{8\beta \gamma^2\theta\log(d/\delta')}{\lambda}}\| \b{w}_* \|_1.
\]
Let $\b{w}_{\mathrm{cs}}$ be the count-sketch estimate of $\b{w}_*$ derived from $\b{w}_{\mathrm{proj}}$, and $\b{w}_{\mathrm{est}}$ be the count-sketch estimate of $\b{w}_*$ derived from $\b{z}_{\mathrm{proj}}$. Recall that the count-sketch estimate of a vector is the median of the estimates of all the locations to which the vector hashes. As the difference between the median of any two vectors is at most the $\ell_{\infty}$-norm of their difference,
\begin{align*}
\lpnorm{\b{w}_{\mathrm{est}}-\b{w}_{\mathrm{cs}}}{\infty}\le \lpnorm{\b{z}_{\mathrm{proj}}-\b{w}_{\mathrm{proj}}}{\infty}.
\end{align*}
Therefore with failure probability $\delta'$,
\begin{align}
\lpnorm{\b{w}_{\mathrm{est}}-\b{w}_{\mathrm{cs}}}{\infty}\le \lpnorm{\b{z}_{\mathrm{proj}}-\b{w}_{\mathrm{proj}}}{\infty}\le  \norm{\b{z}_{\mathrm{proj}}-\b{w}_{\mathrm{proj}}}\nonumber\\
 \le \sqrt{\frac{8\beta \gamma^2\theta\log(d/\delta')}{\lambda}}\| \b{w}_* \|_1. \label{eq:e3}
\end{align}
We now use Lemma \ref{lem:count-sketch} to bound the error for Count-Sketch recovery.
\begin{lemma}\label{lem:count-sketch}
	\citep{charikar2002finding} Let $\b{w}_{\mathrm{cs}}$ be the Count-Sketch estimate of the vector $w$. For any vector $w$, with probability $1-\delta$, a Count-Sketch matrix with width $\Theta(1/\epsilon^2)$ and depth $\Theta(\log(d/\delta))$ satisfies,
	\[
	\lpnorm{\b{w}-\b{w}_{\mathrm{cs}}}{\infty} \le \epsilon \norm{\b{w}}.
	\]
\end{lemma}
Hence using Lemma \ref{lem:count-sketch} for the  matrix $\sqrt{s}{R}$, with failure probability $\delta'$,
\[
\lpnorm{\b{w}_*-\b{w}_{\mathrm{cs}}}{\infty} \le \sqrt{\theta} \norm{\b{w}_*}.
\]
Now using the triangle inequality and Eq. \ref{eq:e3}, with failure probability $2\delta'$ (due to a union bound),
\begin{align*}
\lpnorm{\b{w}_*-\b{w}_{\mathrm{est}}}{\infty}&\le \lpnorm{\b{w}_*-\b{w}_{\mathrm{est}}}{\infty}+\lpnorm{\b{w}_{\mathrm{est}}-\b{w}_{\mathrm{cs}}}{\infty}\\
&\le \sqrt{\theta} \norm{\b{w}_*} + \sqrt{\frac{8\beta \gamma^2\theta\log(d/\delta')}{\lambda}}\| \b{w}_* \|_1\\
&\le \Big( \sqrt{\theta} + \sqrt{\frac{8\beta \gamma^2\theta\log(d/\delta')}{\lambda}}\Big)\| \b{w}_* \|_1.
\end{align*}
Therefore choosing $\theta = \min\{1,\lambda / (4\beta \gamma^2\log(d/\delta'))\}\epsilon^2/4$, with failure probability $2\delta'$,
\[
\lpnorm{\b{w}_*-\b{w}_{\mathrm{est}}}{\infty} \le \epsilon \| \b{w}_* \|_1.
\]
Choosing $\delta'=\delta/2$, for fixed constants $C',C''$ and 
\begin{align*}
k &= {(C'/\epsilon^4)\log^3(d/\delta)\max\{1,\beta^2\gamma^4/\lambda^2\}},\\
s &={(C''/\epsilon^2)\log^2(d/\delta)\max\{1,\beta\gamma^2/\lambda\}},
\end{align*}
with probability $1-\delta$,
\[
\lpnorm{\b{w}_*-\b{w}_{\mathrm{est}}}{\infty} \le \epsilon \| \b{w}_* \|_1.
\]

\subsection{Proof of Theorem \ref{thm:online}}\label{sec:ol_proof}

Let $f_t(\b{z})$ be the loss function corresponding to the data point chosen in the $t$th time step---
\begin{align}
f_t(\mathbf{z})=\ell\left(y_t \mathbf{z}^T {R}\mathbf{x}_t\right) + \frac{\lambda}{2}\|\mathbf{z}\|_2^2. \label{eq:ft}
\end{align}
Let ${\b{z}}_t$ be the weight vector at the $t$th time step for online updates on the projected problem. Let $\bar{\b{z}}=\frac{1}{{T}}\sum_{i=1}^{{T}}\hat{\b{z}}_i$ be the average of the weight vectors for all the ${T}$ time steps. We claim that $\bar{\b{z}}$ is close to ${\b{z}}_{*}$, the optimizer of $\hat{L}(\b{z})$, using Corollary 1 of \citet{shamir2016sgd}. In order to apply the result we first need to define a few parameters of the function $\hat{L}(\b{z})$. Note that $\hat{L}(\b{z})$ is $\lambda$ strongly-convex (as $\hat{L}(\b{z})-\frac{\lambda}{2}\normsq{\b{z}}$ is convex. As the derivative of $\ell$ is bounded above by $H$, therefore $\ell$ is $H$-Lipschitz. We assume $\norm{{R}\b{x}_i}\le B, \norm{\b{z}_*}\le {D}$ and $\max_t \norm{\nabla f_t(\mathbf{w})}\le G$. We will bound $B,D$ and $G$ in the end. We now apply Corollary 1 of \citet{shamir2016sgd}, with the notation adapted for our setting.

\begin{lemma}
	\cite{shamir2016sgd} Consider any loss function $\hat{L}(\b{z})= \sum_{i=1}^{T}f_t(\b{z})$, where $f_t(\mathbf{z})$ is defined in Eq. \ref{eq:ft}. For any $H$-Lipchitz $\ell_i$, $\norm{R\b{x}_i}\le B$, $\norm{\b{z}_t}\le D$, and some fixed constant $C$, over the randomness in the order in which the samples are received:
	\[
	\E\Big[\frac{1}{{T}}\sum_{t=1}^{{T}}\hat{L}({\b{z}}_t) -\hat{L}({\b{z}}_*)\Big]  \le \frac{C(R_T/\sqrt{T} + B{D}H)}{\sqrt{{T}}},
	\]
	where $R_T$ is the regret of online gradient descent with respect to the batch optimizer $\b{z}_*$, defined as $R_T = \sum_{t=1}^{{T}}[f_t(\hat{\b{z}})-f_t(\b{z}_*)]$.
\end{lemma}
By standard regret bounds on online gradient descent (see \citet{zinkevich2003online}), $
R_T \le G{D}\sqrt{T}$. Therefore,
\[
\E\Big[\frac{1}{{T}}\sum_{t=1}^{{T}}\hat{L}({\b{z}}_t) -\hat{L}({\b{z}}_*)\Big]  \le \frac{C{D}(G + BH)}{\sqrt{{T}}}.
\]
Note that by Jensen's inequality,
\begin{align}
&\E[\hat{L}({\b{z}})]\le \E\Big[\frac{1}{{T}}\sum_{t=1}^{{T}}\hat{L}({\b{z}}_t)\Big]\nonumber\\
\implies &\E\Big[\hat{L}(\bar{\b{z}}) -\hat{L}({\b{z}}_*)\Big]  \le \frac{C{D}(G + BH)}{\sqrt{{T}}}.\label{eq:ol_bnd}
\end{align}
We will now bound the expected distance between $\bar{\b{z}}$ and ${\b{z}}_*$ using Eq. \ref{eq:ol_bnd} and the strong convexity of $\hat{L}(\b{w})$. As $\hat{L}(\b{w})$ is $\lambda$ strong-convex and $\nabla\hat{L}(\b{z}_*)=0$, we can write: 
\begin{align*}
&\hat{L}({\b{z}}_*) + ({\lambda}/{2})\normsq{\bar{\b{z}}-{\b{z}}_*}\le \hat{L}(\bar{\b{z}})\\
\implies &\normsq{\bar{\b{z}}-\b{z}_*} \le ({\lambda}/{2})[\hat{L}(\bar{\b{z}})-\hat{L}(\b{z}_*)]\\
\implies &\E\Big[\normsq{\bar{\b{z}}-\b{z}_*}\Big] \le (2/{\lambda})\Big[\E[\hat{L}(\bar{\b{z}})]-\hat{L}(\b{z}_*)\Big].
\end{align*}
Using Eq. \ref{eq:ol_bnd} and then Jensen's inequality,
\begin{align}
\E\Big[\norm{\bar{\b{z}}-\b{z}_*}\Big] \le \frac{2C{D}(G+BH)}{\lambda\sqrt{{T}}}. \label{eq:sgd_edist}
\end{align}
Let $\bar{\b{z}}_{\mathrm{proj}}=\sqrt{s}\bar{\b{z}}$. Let $\b{z}_{\mathrm{wm}}$ is the Count-Sketch estimate of $\b{w}_*$ derived from $\bar{\b{z}}_{\mathrm{proj}}$. Recall from the proof of Theorem \ref{thm:batch-recovery} that ${\b{z}}_{\mathrm{proj}}=\sqrt{s}{\b{z}}$ and $\b{w}_{\mathrm{est}}$ is the Count-Sketch estimate of $\b{w}_*$ derived from ${\b{z}}_{\mathrm{proj}}$. As in the proof of Theorem \ref{thm:batch-recovery}, we note that the difference between the medians of any two vectors is at most the $\ell_{\infty}$ norm of the difference of the vectors, and hence we can write,
\begin{align*}
\lpnorm{\b{w}_{\mathrm{est}}-\b{z}_{\mathrm{wm}}}{\infty}&\le\lpnorm{\b{z}_{\mathrm{proj}}-\bar{\b{z}}_{\mathrm{proj}}}{\infty}\le \norm{\b{z}_{\mathrm{proj}}-\bar{\b{z}}_{\mathrm{proj}}}\\
&= \sqrt{s}\norm{\b{z}_*-\bar{\b{z}}}.
\end{align*}
Therefore, using Eq. \ref{eq:sgd_edist}, 
\begin{align}
\E[\lpnorm{\b{w}_{\mathrm{est}}-\b{z}_{\mathrm{wm}}}{\infty}] \le \frac{2C{D}(G+BH)}{\lambda}\sqrt{\frac{s}{{T}}}.
\end{align}
By the triangle inequality, 
\begin{align}
&\lpnorm{\b{w}_*-\b{z}_{\mathrm{wm}}}{\infty}\le \lpnorm{\b{w}_*-\b{w}_{\mathrm{est}}}{\infty}+\lpnorm{\b{w}_{\mathrm{est}}-\b{z}_{\mathrm{wm}}}{\infty}\nonumber\\
\implies &\E\Big[\lpnorm{\b{w}_*-\b{z}_{\mathrm{wm}}}{\infty}\Big]\le \E\Big[\lpnorm{\b{w}_*-\b{w}_{\mathrm{est}}}{\infty}\Big]+\E\Big[\lpnorm{\b{w}_{\mathrm{est}}-\b{z}_{\mathrm{wm}}}{\infty}\Big]\nonumber
\end{align}
By Theorem \ref{thm:batch-recovery}, for fixed constants $C_1,C_2$ and 
\begin{align*}
k&= {(C_1/\epsilon^4)\log^3(d/\delta)\max\{1,\beta^2\gamma^4/\lambda^2\}},\\
s&={(C_2/\epsilon^2)\log^2(d/\delta)\max\{1,\beta\gamma^2/\lambda\}},\\
\lpnorm{\b{w}_*-\b{w}_{\mathrm{est}}}{\infty} &\le \epsilon \| \b{w}_* \|_1,
\end{align*}
with probability $1-\delta$. Therefore, for fixed constants $C_1'$ and $C_2'$ and probability $1-\delta$,
\begin{align*}
\E\Big[\lpnorm{\b{w}_*-\b{z}_{\mathrm{wm}}}{\infty}\Big] &\le \frac{\epsilon}{2}\lpnorm{\b{w}_*}{1}\\ +& \sqrt{\frac{4C_2'(G{D}+B{D}H)^2\log^2(d/\delta)\max\{1,LR^2/\lambda\}}{\lambda^2 \epsilon^2{T}}}.
\end{align*}
Therefore for 
\[T\ge (C_3'/(\epsilon^4\lambda^2))({D}/\lpnorm{\b{w}_*}{1})^2(G+BH)^2\log^2(d/\delta)\max\{1,LR^2/\lambda\},
\]
\[
\E\Big[\lpnorm{\b{w}_*-\b{z}_{\mathrm{wm}}}{\infty}\Big] \le \frac{\epsilon}{2}\lpnorm{\b{w}_*}{1} +\frac{\epsilon}{2}\lpnorm{\b{w}_*}{2} \le {\epsilon}\lpnorm{\b{w}_*}{1}.
\]
We will now bound $B,D$ and $G$, starting with $B$. Note that ${R}$ is a JL matrix which satisfies condition \ref{eq:ip_preserve} with $\epsilon=\theta$. Using Lemma \ref{lem:jl_l1} and the fact that $\norm{\b{x}_i}\le 1$,
\[
\norm{{R}\b{x}_i} \le \sqrt{1+\theta \gamma^2} \implies B \le 1+\sqrt{\theta}\gamma \le 1+\epsilon\gamma,
\]
where for the last bound we use the setting of 
\[
\theta=\min\{1,\lambda / (4\beta \gamma^2\log(d/\delta'))\}\epsilon^2/4
\]
from the proof of Theorem \ref{thm:batch-recovery}. We next bound $\norm{\b{z}_*}$. Using Lemma \ref{lem:jl_l1},
\begin{align*}
\norm{\b{z}_*-{R}\b{w}_*}\le 2R\theta\sqrt{{\beta}/{\lambda}}\lpnorm{\b{w}_*}{1}\\
\implies \norm{\b{z}_*}\le \norm{{R}\b{w}_*} + 2R\theta\sqrt{{\beta}/{\lambda}}\lpnorm{\b{w}_*}{1}.
\end{align*}
By Lemma \ref{lem:jl_l1}, $\norm{{R}\b{w}_*} \le \sqrt{\normsq{\b{w}_*}+\theta\lpnorm{\b{w}_*}{1}^2}\le \norm{\b{w}_*}+\sqrt{\theta}\lpnorm{\b{w}_*}{1}$. Therefore,
\begin{align*}
\norm{\b{z}_*}&\le \norm{\b{w}_*}+\sqrt{\theta}\lpnorm{\b{w}_*}{1}+ 2R\theta\sqrt{{\beta}/{\lambda}}\lpnorm{\b{w}_*}{1}\\
&=\norm{\b{w}_*}+\Big(\sqrt{\theta}+ 2R\theta\sqrt{{\beta}/{\lambda}}\Big)\lpnorm{\b{w}_*}{1}.
\end{align*}
For our choice of $\theta$, 
\[
\norm{\b{z}_*} \le \norm{\b{w}_*}+\epsilon\lpnorm{\b{w}_*}{1} \implies D \le D_2 + \epsilon D_1.
\]
This implies that the $(D/\lpnorm{\b{w}_*}{1})$ term in our bound for $T$ can be upper bounded by $2D_2/\lpnorm{\b{w}_*}{1}$, yielding the bound on $T$ stated in Theorem \ref{thm:online}. Finally, we need to upper bound $G=\max_t \norm{\nabla f_t(\mathbf{w})}$. We do this as follows:
 \begin{align*}
 {\nabla f_t(\mathbf{z})} &= {\ell'(y_t\b{z}_t^T{R}\b{x_t})A\b{x_t}+\lambda \b{z}_t}\\
 \implies \norm{\nabla f_t(\mathbf{z})} &\le {|\ell'(y_t\b{z}_t^T{R}x)|\norm{{R}x}+\lambda \norm{\b{z}_t}}\\
 &\le H(1+\epsilon\gamma) + \lambda D.
 \end{align*}
\section{$k$-Independence of Hash Functions}


Our analysis of the WM-Sketch requires hash functions that are $O(\log(d/\delta))$-wise independent. While hash functions satisfying this level of independence can be constructed using polynomial hashing \citep{carter1977universal}, hashing each input value would require time $O(\log(d/\delta))$, which can be costly when the dimension $d$ is large. Instead of satisfying the full independence requirement, our implementation simply uses fast, 3-wise independent tabulation hashing. In our experiments, we did not observe any significant degradation in performance from this choice of hash function.
\newpage
\section{Baseline Algorithms}
\label{sec:baseline-algs}

Here we give pseudocode for the simple truncation and probabilistic truncation baselines evaluated in our experiments.

\begin{algorithm}
  \SetAlgoLined
  \SetAlgoNoEnd
  \DontPrintSemicolon
  
  \newcommand\mycommfont[1]{\rmfamily{#1}}
  \SetCommentSty{mycommfont}
  \SetKwComment{Comment}{$\triangleright$ }{}
  
  \Input{loss function~$\ell$, budget~$K$, $\ell_2$-regularization parameter~$\lambda$, learning rate schedule~$\eta_t$}
  
  \Init{}{
  	  $S\leftarrow \{\}\quad$ \Comment*[l]{Empty heap}
      $t \leftarrow 0 $\;
  }

  \SetKwFunction{FDot}{Dot}
  \SetKwFunction{FTruncate}{Truncate}
  \SetKwFunction{FUpdate}{Update}
  \Fn{\FUpdate{$\mathbf{x}$, $y$}}{
  	$\tau \leftarrow \sum_{i \in S} S[i] \cdot x_i\quad$ \Comment*[f]{Make prediction}
    $S \leftarrow (1 - \lambda \eta_t)S - \eta_t yx_i \nabla\ell(y\tau)$ \;
	Truncate $S$ to top-$K$ entries by magnitude\;
    $t \leftarrow t + 1$\;
  }

  \caption{Simple Truncation}
  \label{alg:truncation}
\end{algorithm}

\begin{algorithm}
  \SetAlgoLined
  \SetAlgoNoEnd
  \DontPrintSemicolon
  
  \newcommand\mycommfont[1]{\rmfamily{#1}}
  \SetCommentSty{mycommfont}
  \SetKwComment{Comment}{$\triangleright$ }{}
  
  \Input{loss function~$\ell$, budget~$K$, $\ell_2$-regularization parameter~$\lambda$, learning rate schedule~$\eta_t$}
  
  \Init{}{
  	  $S_0 \leftarrow \{\}$ \Comment*[f]{Empty heap}
	   
	  $W \leftarrow \{\}\quad$ \Comment*[f]{Reservoir weights}
	  
	  $t \leftarrow 0 $\;
  }

  \SetKw{KwAnd}{and}
  \SetKwFunction{FTruncate}{Truncate}
  \SetKwFunction{FUpdate}{Update}
  \Fn{\FUpdate{$\mathbf{x}$, $y$}}{
  	$\tau \leftarrow \sum_{i \in S_t} S_t[i] \cdot x_i\quad$ \Comment*[f]{Make prediction}
	
	$S_{t+1} \leftarrow (1 - \lambda \eta_t)S_t - \eta_t y\mathbf{x}\nabla\ell(y\tau)$\;
    \For{$i \in S_{t+1}$}{    
      \eIf{$i \not\in S_t$}{
    	$r \sim \mathcal{U}(0, 1)$\;
	
		$W[i] \leftarrow r^{1 / |S_{t+1}[i]|}$ \Comment*[f]{New reservoir weight}
      }{
        $W[i] \leftarrow W[i]^{|S_t[i] / S_{t+1}[i]|}$ \Comment*[f]{Update weight}
      }
    }
	Truncate $S_{t+1}$ to top-$K$ entries by reservoir weight\;
    $t \leftarrow t + 1$\;
  }

  \caption{Probabilistic Truncation}
  \label{alg:prob-truncation}
\end{algorithm}

\end{document}